%% file: main.tex
\documentclass{article}
\usepackage{bbm,bm}
\usepackage{amsmath}
\allowdisplaybreaks
\usepackage{amssymb}
\usepackage{mathtools,enumitem}
\usepackage{amsthm}
\usepackage{microtype}
\usepackage{graphicx}
\usepackage{subcaption}
\usepackage{booktabs} % for professional tables
\newcommand{\E}{\mathbb{E}}
\newcommand{\R}{\mathbb{R}}

\newcommand{\prob}{\mathbb{P}}
 
\newcommand{\greedy}{\text{greedy}}

\newcommand{\step}{\mathrm{K-1}}
\newcommand{\steps}{\mathrm{K}}
\usepackage{hyperref}
\hypersetup{
    colorlinks=true,       % Enable colored links
    linkcolor=blue,        % Color for internal links (e.g., sections)
    citecolor=blue,        % Color for citations
    filecolor=blue,        % Color for file links
    urlcolor=blue          % Color for URLs
}

\usepackage[accepted]{icml2026}

\theoremstyle{plain}
\newtheorem{theorem}{Theorem}[section]

\newtheorem{lemma}[theorem]{Lemma}
\newtheorem{corollary}[theorem]{Corollary}
\theoremstyle{definition}
\newtheorem{definition}[theorem]{Definition}
\newtheorem{assumption}[theorem]{Assumption}
\theoremstyle{remark}
\newtheorem{remark}[theorem]{Remark}

\usepackage[textsize=tiny]{todonotes}

\icmltitlerunning{K-Step Lookahead Thresholding}

\begin{document}

\twocolumn[
  \icmltitle{
  Fast Non-Episodic Finite-Horizon RL with K-Step Lookahead Thresholding
  % K-Step Lookahead Is Enough: Accelerated Convergence in Non-Episodic Finite-Horizon RL via K-Step Lookahead Thresholding Policies
  }



  \icmlsetsymbol{equal}{*}

  \begin{icmlauthorlist}
    \icmlauthor{Jiamin Xu}{yyy}
    \icmlauthor{Kyra Gan}{yyy}
    % \icmlauthor{Firstname3 Lastname3}{comp}
  \end{icmlauthorlist}

  \icmlaffiliation{yyy}{ORIE, Cornell Tech, New York, USA}

  \icmlcorrespondingauthor{Kyra Gan}{kyragan@cornell.edu}

  \icmlkeywords{Non-episodic RL, Finite Horizon, K-step Lookahead, Sample Efficiency}

  \vskip 0.3in
]



\printAffiliationsAndNotice{}  % no special notice (required even if empty)

% Kyra: in discussion, remember to include phrases like in the bandit approximation, thresholding objective is not the key, but rather LCB, which has been proven to be regret optimal for thresholding bandit, empirically performs well in the MDP setting where making a mistake at the beginning of the horizon is costly

\begin{abstract}
% Despite rapid reinforcement learning advances in episodic and infinite-horizon settings, non-episodic finite-horizon MDPs---critical for clinical, financial, and operational decisions---remain underexplored. 
% In these settings, optimal policies depend fundamentally on end-of-horizon returns, but directly learning such long-term estimates from online interaction is prohibitively sample-inefficient.

% Despite recent advancements, o
Online reinforcement learning in non-episodic, finite-horizon MDPs remains underexplored and is 
% fundamentally
challenged by the need to estimate returns to a fixed terminal time.
Existing infinite-horizon methods, 
which often rely on discounted contraction,
do not naturally account for this fixed-horizon structure. 
% To address this, w
We introduce a modified Q-function: rather than targeting the full-horizon, we learn a K-step lookahead Q-function that truncates planning to the next K steps.
% This enables tractable learning while preserving alignment with the finite-horizon objective.
To further improve sample efficiency, we introduce a thresholding mechanism: actions are selected only when their estimated K-step lookahead value exceeds a time-varying threshold. We provide an efficient tabular learning algorithm for this novel objective, proving it achieves fast finite-sample convergence: it achieves minimax optimal constant regret for $K=1$ and $\mathcal{O}(\max((K-1),C_{K-1})\sqrt{SAT\log(T)})$ regret for any $K \geq 2$. 
We numerically evaluate the performance of our algorithm under the objective of maximizing reward.
Our implementation adaptively increases K over time, balancing lookahead depth against estimation variance. Empirical results demonstrate superior cumulative rewards over state-of-the-art tabular RL methods across synthetic MDPs and RL environments:  JumpRiverswim, FrozenLake, and AnyTrading. Code is provided on \href{https://github.com/jamie01713/K-Step-Lookahead}{github}.
% \kyra{feel that we should mention contextual bandit somewhere}
  % This work addresses the sample inefficiency of online reinforcement learning in non-episodic, finite-horizon MDPs. We argue that faster convergence is paramount for finite-sample performance in this setting. 
  % Although the optimal policy depends on \emph{end-of-horizon lookahead thresholding}, learning it directly is prohibitively slow. 
  % Instead, we target approximate \emph{K-step lookahead thresholding} policies, which enable accelerated convergence. We show that learning such a policy is equivalent to minimizing regret in thresholding contextual bandit with K-step lookahead reward. Based on this reduction, we propose an efficient algorithm that achieves \emph{minimax constant regret} for $K=1$ and \emph{sublinear regret} for any $K\geq 2$. 
  % Empirically, our method outperforms state-of-the-art RL finite-and infinite-horizon baselines on Riverswim and across 1000 random MDPs. We also numerically show that our method serves as an effective warm-start to accelerate other RL algorithms on Riverswim.
\end{abstract}
\input{intro}
% \vspace{-5pt}
\section{Problem Setup: Finite Horizon MDP}
% \vspace{-5pt}
% and K-Step Lookahead Policies}
% \subsection{Finite Horizon MDP}
\label{subsec:finite-horizon-mdp}
% In
We consider finite-horizon MDPs with discrete states and actions. 
% The learner is provided with 
An MDP characterized by the tuple $(\mathcal{S},\mathcal{A},P_a,R)$, where $P_a(s,s'): \mathcal{S}\times\mathcal{S}\to [0,1]$ denotes the transition probability from state $s$ to state $s'$ under action $a$. We will use $P_a(s): \mathcal{S}\to\mathit\Delta(\mathcal{S})$ where $\mathit\Delta(\mathcal{S})$ denotes the probability simplex over all possible states to denote the distribution of next state under state $s$ and action $a$. $R(s,a): \mathcal{S}\times\mathcal{A}\to \R$ denotes the \emph{expected reward} that a agent gets when in state $s$ and takes action $a$. We will use $R_{s,a}$ to denote $R(s,a)$ for simplicity. We will use $S:=|\mathcal{S|}$ to denote the cardinality of $\mathcal{S}$ and $A:=|\mathcal{A}|$ to denote the cardinality of $\mathcal{A}$.

At each time $t$, given
the state $s_t$, 
the learner's decision $a_t$ follows a history-dependent random policy, $\pi_{t}$. 
Let
$\mathcal{H}_{t}:=\{(s_i,a_i,r_i)\}_{i=0}^{t-1}\cup \{s_t\}$ 
% $\mathcal{H}_{t}:=\{(s_i,a_i,r_i,s_{i+1})\}_{i=0}^{t-1}$ 
be the history observed up to time $t$. Similarly, denote $\mathcal{F}_t$ as the $\sigma$-algebra generated by $\mathcal{H}_t$. 
Then, 
$\pi_{t}:\mathcal{H}_{t}\to\mathit{\Delta}{\left(\mathcal{A}\right)}$, where
$\mathit{\Delta}{(\mathcal{A})}$ denotes the probability simplex over all possible actions. 
After pulling arm $a_t$, the learner receives a feedback $r_t$ and
observes the next state 
$s_{t+1}\sim P_{a_{t}}\left(s_t,s_{t+1}\right)$.
Suppose that the reward $r_t$ is generated according to a distribution $\mathcal{P}_t$, we will assume stationarity of the \emph{expected} reward received at time $t$, $r_{t}$, when conditioned on history $\mathcal{H}_{t}$ and the current action.
\begin{assumption}\label{assump:stationarity}
    $\E\left[r_{t}\mid a_t, \mathcal{H}_t\right]=R_{s_t,a_t}.$
\end{assumption}
We permit non-stationarity in the rewards, provided that Assumption~\ref{assump:stationarity} is satisfied, offering a relaxation of the standard stationarity assumption typically imposed in literature. 

Given a fixed horizon of length $T$, the goal of the learner is to find the policy $\bm{\pi}:=\{\pi_t\}_{t\in[T]}$ that maximizes the expected cumulative reward, where the expectation is further taken over the randomized policy ($a_{t}\sim \pi_{t}$), in addition to the states ($s_{t+1}\sim P_{a_{t}}(s_t,s_{t+1})$): 
\begin{align}
% \vspace{-15pt}
    \max_{{\bm\pi}}\; &\E_{\bm{\pi}}\left[\sum_{t=0}^TR_{s_t,a_t}\mid{s}_0
    % \mid \bm{s}_0=(0,0,\cdots,0)
    \right]
  .
  \label{eq:optimization-problem-online}
\end{align}
We will use $\bm\pi^*$ to denote the optimal policy corresponding to Problem~\eqref{eq:optimization-problem-online}. 

We use $V_h^{\bm\pi}(s):\mathcal{S}\to\R$ to denote the expected sum of
remaining rewards received under policy $\bm\pi$ starting from state $s$ at time $h$ until the end of the horizon:
% \begin{equation*}
    $V_h^{\bm\pi}(s):=\E_{\bm\pi}\left[\sum_{t=h}^T R_{s_t,a_t}\mid s_h=s\right].$
% \end{equation*}
Similarly, we use $Q_h^{\bm\pi}(s,a):\mathcal{S}\times\mathcal{A}\to\R$ to denote the expected sum of
remaining rewards received under policy $\bm\pi$ starting from state $s$ at time $h$ until the end of the horizon:
% \begin{equation*}
    $Q_h^{\bm\pi}(s,a):=\E_{\bm\pi}\left[\sum_{t=h}^T R_{s_t,a_t}\mid s_h=s,a_h=a\right].$
% \end{equation*}
Let $V_h^*:=\max_{\bm\pi}V_h^{\bm\pi}$ and $Q_h^*:=\max_{\bm\pi}Q_h^{\bm\pi}$. We have the following Bellman optimality equations:
\begin{align*}
    V_h^*(s)&=\max_aQ_h^*(s,a),\forall s\in\mathcal{S}\\
    Q_h^*(s,a)&=R_{s,a}+\E_{s'\sim P_a(s,\cdot)}\left[V_{h+1}(s')\right],\forall s\in\mathcal{S}.
\end{align*}
For simplicity, we assume that at each time, the
% there exists an \emph{unique} 
action $a_{s,t}^{*}$ that maximizes $Q_{T-t}^*(s)$ is \emph{unique} for all $t\in[T]$. The optimal policy
% Since $\mathcal{S},\mathcal{A}$ are finite and the horizon is finite, the following policy $\bm\pi$ is the optimal policy, which we will use 
$\bm\pi^*$ can be described as:
% to denote:
\begin{equation}\label{eq:end-of-horizon-policy}
    \pi_t^*(a_t|s)=\begin{cases}
        1&a_t=\arg\max_a Q_t^*(s,a)\\
        0&\text{otherwise}
    \end{cases}.
\end{equation}
% \vspace{-5pt}
\subsection{Hardness of Learning Finite Horizon MDP Online}\label{subsec:hardness} 
% \vspace{-5pt}
Both model-based \citep{10.5555/3305381.3305409} and model-free methods \citep{jin2018q} depend on estimating $Q_h^*$. However, since the $Q$-function in a finite horizon is measuring the cumulative returns until the end of the horizon, it is sample-inefficient to learn in a finite horizon, non-episodic setting. The best known regret upper bound \citep{10.5555/3305381.3305409,zhang2020almost} for episodic setting translates to a linear regret in a non-episodic setting, where the regret is defined as:
% \begin{align}
% \label{eq:regret-rl}
    % :&=
$$\mathcal{R}_{Q}^{\bm\pi^*}(\bm\pi):= \E_{\bm{\pi}^*}\left[\sum_{t=0}^TR_{s_t,a_t}\mid{s}_0
    % \mid \bm{s}_0=(0,0,\cdots,0)
    \right]
    % \notag\\&
    -\E_{\bm{\pi}}\left[\sum_{t=0}^TR_{s_t,a_t}\mid{s}_0
    % \mid \bm{s}_0=(0,0,\cdots,0)
    \right].$$
% \end{align}
Furthermore, the regret lower bound is also $\Omega(T)$
\citep{10.5555/3305381.3305409,NIPS2008_e4a6222c}, meaning that all online algorithms suffer from a linear convergence rate to $\bm\pi^*$. 
\begin{remark}[Hardness of Applying Infinite-horizon Algorithms]
% Although maximizing infinite-horizon average reward can be seen as an approximation for Problem~\eqref{eq:optimization-problem-online} when the horizon $T$ is large,
While infinite-horizon average-reward maximization can serve as an approximation to Problem~\eqref{eq:optimization-problem-online} for large horizons $T$,
existing online algorithms designed for this setting rely on structural assumptions that guarantee
% either
% that maximize the infinite-horizon reward rely either on
convergence to 
a stationary distribution \citep{wei2020model} or 
quick recovery to high-reward states
% recovering to high-reward states relatively quickly 
\citep{boone2024achieving}. 
In a short, finite horizon, these underlying assumptions 
% the above-mentioned two properties 
are not satisfied,
and thus the guaranteed properties do not hold.
As we demonstrate empirically in Section~\ref{sec:experiments}, this mismatch leads to slow convergence for such algorithms in our setting.
% and we will show empirically that those algorithms still suffer from a slow convergence rate in Section~\ref{sec:experiments}. 
\end{remark}

% \vspace{-5pt}
\section{K-step Lookahead Q-function and K-step Lookahead Thresholding} \label{subsec:k-step}
% \vspace{-5pt}
As mentioned above,
% in Section~\ref{subsec:finite-horizon-mdp}, 
learning $\bm\pi^*$ suffers from slow convergence
% a linear lower bound for convergence rate
because of the need to estimate $Q_h^*$. To address this problem, we will make two key modifications:
\begin{enumerate}[leftmargin=*, itemsep=0pt, topsep=0pt, partopsep=0pt, parsep=1pt]
    \item \emph{K-step Lookahead Q-function}: We target a K-step lookahead Q-function $Q_{T-K}^*$. This means instead of estimating the return until the end of the horizon, we only focus on learning the return in K steps. By truncating the horizon, sample complexity can be improved. 
    % In the special case w
    When K $=1$, the problem becomes a contextual bandit where context follows an MDP,
    which is known to have a lower sample complexity
    % . This reduces sample complexity when compared with MDPs 
    \citep{simchi-xu-bypass}. 
    \item \emph{Thresholding objective}:
    At each step, we identify actions that the K-step lookahead reward exceeds a threshold (vs. optimizing cumulative rewards), further lowering sample complexity \citep{feng2025satisficingregretminimizationbandits} to achieve fast finite-horizon convergence.
\end{enumerate}
To formalize this approach, we first define the core quantity used for K‑step planning: the lookahead reward.
\begin{definition}[K-Step Lookahead Reward]
The K-step lookahead reward $r^\steps_{s,a}$ is the maximum total expected reward attainable from state $s$ and action $a$ over the next $K$ steps. Due to the time-homogeneity of the MDP, this equals the optimal Q-value at step 
$T-K+1$ (with exactly 
$K$ steps remaining): $r^\steps_{s,a}:=Q_{T-K+1}^*(s,a)$.
    % Let K-step lookahead reward $r^\steps_{s,a}$ be defined as:
    % \begin{equation*}
    %     r^\steps_{s,a}:=Q_{T-K}^*(s,a).
    % \end{equation*}
    % This represents the maximum total reward achievable starting from $(s,a)$ over the final $K$ steps of the horizon.
    % where $Q_{T-K}^*(s,a)$ is the optimal Q-function at time $T-K$.
\end{definition}
We denote by $a_{s,K}^*$ the action that maximizes $r_{s,a}^{K}$: $a_{s,K}^*:=\arg\max_ar_{s,a}^\steps$ for $K\geq 1$.
% \kyra{fix notation, range of $l$? and uniqueness. I also don't like this notation $r_{s,\cdot}^{l}$}

% Specifically, by truncating the planning of Q-function to K steps, we focus on  K-step lookahead greedy policy $\bm\pi^{\steps,\greedy}$ which is choosing the action that maximizes the K-step lookahead reward when the remaining horizon $h>K$ and choosing the action that maximizes the h-step lookahead reward when the remaining horizon $h\leq K$. 
% Before we mathematically define the K-step lookahead greedy policy, We first define K-step lookahead reward 
% \begin{definition}[K-Step Lookahead Reward]
%     Let K-step lookahead reward $r^\steps_{s,a}$ be defined as:
%     \begin{equation*}
%         r^\steps_{s,a}:=Q_{T-K}^*(s,a),
%     \end{equation*}
%     where $Q_{T-K}^*(s,a)$ is the optimal Q-function at time $T-K$.
% \end{definition}
% We will use $a_{s,l}^*$ to denote the action that maximizes $r_{s,\cdot}^{l}$. 
% \kyra{$r_{s_t,\cdot}^{\min(\mathrm{T-t+1,K})}$: your prior state and action associated with reward had always been in (), keep the notation consistent and need to explain this notation in word first as it is complicated}
A K‑step lookahead greedy policy $\bm\pi^{K,\text{greedy}}$ makes decisions by looking ahead at most $K$ steps. When the remaining horizon $h = T-t+1$ is greater than $K$, it chooses the action maximizing the $K$-step reward $r^\steps_{s,a}$ with respect to $a$ . When $h \leq K$, it looks ahead to the terminal step by maximizing the $h$-step reward $r^{\mathrm{h}}_{s,a}$ with respect to $a$. This logic is captured succinctly by maximizing $r^{\min(\mathrm{h}, K)}_{s,a}$ with respect to $a$.
Formally, $\bm\pi^{\steps,\greedy}$ is defined as:
\begin{equation}\label{eq:K-step-lookahead-greedy-policy}
    \pi_t^{\steps,\greedy}(a_t|s)=\begin{cases}
        1&a_t=\arg\max_a r^{\min(\mathrm{h},
         \mathrm{
        % T-t+1,
        K
        }
        )}_{s,a}\\
        0&\text{otherwise}
    \end{cases}.
\end{equation}
% \kyra{missing superscripts $K,$greedy in above equation}

% \kyra{okay, the sequence of things are defined are flipped. need to define $Q^*$ and $r$ first}

To further reduce the sample complexity, we target at K-step lookahead thresholding policy, denoted by $\bm\pi^{\steps,\bm\gamma}$.
At each time step $t$, this policy selects an action uniformly at random from those satisfying
% which chooses actions that satisfy
$r_{s_t,a}^{\min{(h,
% \mathrm{T-t+1},
\steps)}}\geq \gamma_t$.
If no such action exists, it instead selects $a_{s,\min(h,K)}^*$. 
% \kyra{the notation for maximizer should have been defined above}.
% when no such actions exist.
We note that by setting the thresholds  $\bm\gamma:=\{\gamma_t\}_{t\in[T]}$ sufficiently high, this policy reduces to a greedy policy. 
% when we choose $\bm\gamma$ high enough, the thresholding policy will be the same as greedy policy. 
% \vspace{-5pt}
\paragraph{Regret as Convergence} 
Following prior work \cite{xu2025restlesscontextualthresholdingbandit,candelieri2023mastering}, 
% \kyra{add additional work that have done this work},
we use regret to measure the convergence rate of online algorithms to the oracle policy that they are converging to. 
Since our algorithm converges to the K-step lookahead thresholding policy $\bm\pi^{K,\gamma}$ and not to the optimal policy $\bm\pi^*$, the regret is consequently defined as:
% Since we are targeting different policies that is not $\bm\pi^*$, we will use the following definition of regret instead of using $\mathcal{R}_Q^{\bm\pi^*}$ to measure convergence rate to $\bm\pi^{\steps,\bm\gamma}$:
\begin{align}
    \mathcal{R}^{\bm\pi^{\steps,\bm\gamma}}(\bm\pi):&=\E_{\bm\pi^{\steps,\bm\gamma}}\left[\sum_{t=0}^Tc(s_t,a_t)\mid s_0\right]\notag\\&-\E_{\bm\pi^{\steps,\bm\gamma}}\left[\sum_{t=0}^Tc(s_t,a_t)\mid s_0\right],\label{eq:regret-threshold}
\end{align}
where cost function $c(s_t,a_t)$ is defined as:
\begin{equation*}
    c(s_t,a_t):=(\gamma_t-r^{\min{(\mathrm{h
    % T-t+1
    },\steps)}}_{s_t,a_t})\mathbb{I}\left\{r^{\min{(\mathrm{h
    % T-t+1
    },\steps)}}_{s_t,a_t}<\gamma_t\right\}.
\end{equation*}
% This cost function penalizes when the agent chooses a
This cost function penalizes the selection of any 
``bad'' action
whose K-step lookahead reward falls below the threshold.
% that yields a K-step lookahead reward that is smaller than the threshold. 
Thus, Eq.~\eqref{eq:regret-threshold} measures how frequently a policy $\bm\pi$ deviates from the K‑step lookahead thresholding policy $\bm\pi^{\steps,\bm\gamma}$, which in turn qualifies 
the rate of convergence toward $\bm\pi^{\steps,\bm\gamma}$.
% the convergence rate to the K-step lookahead thresholding policy. 
% We will use $\Delta_{s,a,t}^\steps$ to denote $\gamma_t-r_{s,a}^\steps$ and $\Delta^\steps$ to denote $\min_{s,a,t}|\Delta_{s,a,t}|$.

% \kyra{here}
% \vspace{-5pt}
\subsection{Optimality of K-Step Lookahead Greedy Policy}\label{subsec:optimality}
% \vspace{-5pt}
When $K$ exceeds the horizon $T$, i.e., $K\geq T$, the greedy policy $\bm\pi^{\steps,\greedy}$ conincides with the optimal policy $\bm\pi^*$ by Eq.~\eqref{eq:K-step-lookahead-greedy-policy}, and is therefore optimal. 
% \kyra{cite and state the worst-case result}
% By definition of $\bm\pi^{\steps,\greedy}$ (Eq.~\eqref{eq:K-step-lookahead-greedy-policy}), when $K\geq T$, we have $\bm\pi^{\steps,\greedy}=\bm\pi^*$ according to Eq.~\eqref{eq:end-of-horizon-policy}. This means that our K-step lookahead greedy policy is \emph{optimal} when $K$ large enough. 

In the rest of the subsection, we first establish 
% The rest of the subsection establishes
that $\bm\pi^{\steps,\greedy}$ is always optimal for a \emph{two-state} MDP under
the stochastic dominance assumption (Assm.~\ref{assumption:stochastic-dominance}) in Thm. \ref{th:optimality-two-state}. Next, we show that in the general state setting, $\bm\pi^{\steps,\greedy}$ incurs an optimality gap that is linear in $T$ in the worst case (Thm.~\ref{th:lower-bound-linear}).
% the following stochastic dominance assumption.
% stated below.
\begin{assumption}[Stochastic Dominance Assumption]\label{assumption:stochastic-dominance}
Consider an MDP with a binary state space $\mathcal{S} = \{0,1\}$, where state $1$ yields higher maximum rewards: $\max_a r_{1,a} \geq \max_a r_{0,a}$. Let $a_{s,1}^*$ denote the action maximizing the one-step lookahead reward from state $s$. We assume that for every state $s \in \mathcal{S}$ and every action $a \in \mathcal{A}$,
    % Let the state space of MDP $\mathcal{S}$ be binary $\{0,1\}$ and $\max_a r(1,a)\geq \max_ar(0,a)$. Assume the following holds for all action $a\in\mathcal{A}$ and $s\in\mathcal{S}$:
    % \begin{equation*}
        $P_{a_{s,1}^*}(s,1)\geq P_{a}(s,1).$
    % \end{equation*}
\end{assumption}
Intuitively, Assumption~\ref{assumption:stochastic-dominance} guarantees that the action with the highest immediate reward also maximizes the probability of transitioning to the more rewarding state. This alignment ensures that short-term greedy decisions do not compromise long-term value. In healthcare decision-making, Assumption~\ref{assumption:stochastic-dominance} is satisfied:  higher treatment intensity (action) can both increase immediate reward (e.g., higher recovery probability) and shift the next-state distribution toward healthier states. Theorem~\ref{th:optimality-two-state} establishes the optimality of 
% taking an action that has high immediate reward will not transit to a state with low reward. 
% Therefore, under Assumption~\ref{assumption:stochastic-dominance}, we can show
$\bm\pi^{\steps,\greedy}$ 
% is optimal 
for any $K\geq 1$.
\begin{theorem}\label{th:optimality-two-state}
    Under Assumption~\ref{assumption:stochastic-dominance} and binary states,
    % assume $\mathcal{S}=\{0,1\}$,
    $\bm\pi^{\steps,\greedy}=\bm\pi^*$ for any $K\geq 1$. 
\end{theorem}
The proof (Appendix~\ref{proof:optimality-two-state}) relies on the use of
% We will use 
inductions to show that $V_t^*$ is a monotone function and then use Assumption~\ref{assumption:stochastic-dominance} to show the optimality. 
% The details are deferred to Appendix~\ref{app:optimality}.

For general states, we show that there always exists an MDP in which the $K$-step lookahead thresholding policy incurs a linear optimality gap for any fixed $K < T$.
% such that $\bm\pi^{\steps,\greedy}$ has an optimality gap that is linear in $T$ for any fixed $K<T$. Formally,
\begin{theorem}\label{th:lower-bound-linear}
    For any $K<T, \mathcal{S},\mathcal{A}$, there exists an MDP instance with
    % such that there exists 
    a state $s\in\mathcal{S}$ such that $V_0^{\bm{\pi}^*}(s)-V_0^{\bm{\pi}^{\steps,\bm\gamma}}(s)=\Theta(T)$.
\end{theorem}
The proof of Thm.~\ref{th:lower-bound-linear} (Appendix~\ref{proof:lower-bound-linear}) relies on 
% We provide a sketch of 
constructing a bad instance where the beneficial long-term action carries a large immediate penalty. 
While taking this penalty is optimal over the full horizon, a 
K-step lookahead policy (with 
$K<T$) will always reject it due to the short-term cost, thereby remaining trapped in a low-value state. This results in a per-step loss that accumulates linearly with $T$.

To better characterize the instance-dependent constant of the linear suboptimality gap, we show that the suboptimality gap is dependent on the $\mathcal{L}_1$ distance of transition probabilities.
\begin{theorem}\label{th:instance-lower-bound}
    For any $K<T, \mathcal{S},\mathcal{A}$, 
    \begin{align}
        &V_0^*(s)-V_0^{\bm\pi^{\steps,\bm\gamma}}(s)\\&\leq \sum_{t=0}^{T-K}\max_{s'}\left|P_{\pi_t^{\steps,\bm\gamma}}\left(s'\right)-P_{a_{s',T-t+1}^*\left(s'\right)}\right|_1\max_{s'}V_{t+1}^*\left(s'\right),\notag
    \end{align}
where $P_{\pi_t^{\steps,\bm\gamma}}\left(s'\right):=\sum_a\pi_{t}^{\steps,\bm\gamma}\left(a|s'\right)P_{a}\left(s'\right)$.
\end{theorem}
The proof of Theorem~\ref{th:instance-lower-bound} (Appendix~\ref{sec:proof-instance}) relies on backward induction on the difference of value function. Theorem~\ref{th:instance-lower-bound} shows that the K-step lookahead thresholding policy can achieve near optimal when the transition matrix under different actions are close in $\mathcal{L}_1$ distance.
% The state space consists a good state, a bad state and several dummy states that behave the same as the good state. From the bad state, action $a_0$ will make it stay at bad state with a small negative reward while action $a_1$ will make it go to good state but with a large negative reward. Though going to the good state will be optimal, it can not be identified with K step lookahead because of the immediate large negative reward. 
% Details are referred to Appendix~\ref{proof:lower-bound-linear}. 
% We will first provide a sketch of constructing the bad instance: The state space consists of a \emph{good} state, a \emph{bad} state, and several passive dummy states. From the bad state, the agent has two choices: action \(a_0\), which yields a small negative reward and keeps the agent in the bad state with high probability, or action \(a_1\), which yields a large negative reward but transitions the agent to the good state with high probability. From the good state, all actions yield zero reward and keeps in the good state with high probability. Consequently, while the optimal policy must eventually select \(a_1\) to escape the bad state, doing so requires long-term planning to overcome the large immediate penalty. 
% The full details of the proof are provided in Appendix~\ref{proof:lower-bound-linear}.\kyra{shorten proof sketch}

Before we introduce our online learning algorithm in Section~\ref{sec:online-algorithm}, we note that a fundamental trade-off exists between an algorithm's convergence rate and the optimality of its target policy. While Theorems~\ref{th:lower-bound-linear} and \ref{th:instance-lower-bound} shows that $\pi^{\steps,\bm\gamma}$ can have linear suboptimality, Section~\ref{sec:online-algorithm} will demonstrate that an online algorithm can achieve sublinear regret against this oracle. This stands in contrast to the known linear regret lower bound when competing with the optimal policy (Section~\ref{subsec:finite-horizon-mdp}).
% Before we end this section, we note that a fundamental trade-off exists between the convergence rate of an online algorithm and the optimality of its oracle policy. While we have established that $\bm\pi^{\steps,\greedy}$ can perform linearly poorly in the worst case (Thm.~\ref{th:lower-bound-linear}), we will show in Section~\ref{sec:online-algorithm} that an online algorithm achieving a sublinear convergence rate to this oracle is attainable. This contrasts with the known linear lower bound on regret against the optimal policy as discussed in Section~\ref{subsec:finite-horizon-mdp}. Hence, one must balance the speed of convergence against the quality of the benchmark. 
We argue that in a finite horizon, a fast convergence rate is the key to achieving high cumulative reward---a claim that we empirically validate in Section~\ref{sec:experiments}.
% \vspace{-5pt}
\section{Online K-Step Lookahead}\label{sec:online-algorithm}
% \vspace{-5pt}
% \kyra{here}
We begin by presenting our learning algorithm for the 1‑step lookahead thresholding policy (Alg.~\ref{alg:'name'}, Sec.~\ref{subsec:one-step}).
We establish that Algorithm~\ref{alg:'name'} achieves a minimax optimal constant regret (Thm.~\ref{th:one-step-regret}).
Next, we extend it to the general K‑step case (Alg.~\ref{alg:k-step}, Sec.~\ref{subsec:online-k-threshold-algo}). We establish that Algorithm~\ref{alg:k-step} achieves
% we obtain
a regret bound of  $\max(\mathcal{O}((K-1)\sqrt{SAT}),\mathcal{O}(C_{K-1}\sqrt{SAT\log(T)}))$ (Thm.~\ref{th:k-step-regret}).
These results demonstrate a fast convergence rate for our method when $K$ is small in the non-episodic finite-horizon setting (Remark~\ref{remark:impact-of-k-regret}).
% We now first present our learning algorithm (Algorithm~\ref{alg:'name'}) for 1-step lookahead thresholding policy and then extend to Algorithm~\ref{alg:k-step} for more general K-step lookahead thresholding policy. We establish a minimax optimal constant regret rate for 1-step lookahead and a $\max(\mathcal{O}((K-1)\sqrt{SAT}),\mathcal{O}(C_{K-1}\sqrt{SAT\log(T)}))$ for K-step lookahead. This demonstrates the fast convergence rate of our algorithm in non-episodic, finite horizon setting. We note that 
In this section, 
we \emph{do not require} Assumption~\ref{assumption:stochastic-dominance} or binary state. 
% \vspace{-5pt}
\subsection{Online One-Step Lookahead}\label{subsec:one-step}
% \vspace{-5pt}
When $K=1$, our problem reduces to a standard thresholding contextual bandit problem.
Algorithm~\ref{alg:'name'} selects actions for the one-step lookahead policy via a lower confidence bound (LCB) test against the threshold $\bm\gamma$. At each time $t$:
\begin{enumerate}[leftmargin=12pt, itemsep=0pt, topsep=0pt, partopsep=0pt, parsep=1pt]
    \item It constructs a candidate set $\tilde{G}_t$ of actions whose LCB for the one-step reward is at least $\gamma_t$.
    \item If $|\tilde{G}_t| \geq 1$, it plays the action from this set with the highest LCB.
    \item Otherwise, it plays an action uniformly at random.
\end{enumerate}

% Compared with prior results, we achieve a minimax optimal constant regret by introducing a novel design of lower confidence bound (Remark~\ref{remark:improved-convergence}).

% \kyra{missing discussion on reduction to thresholding contextual bandit, and highlight contribution: algorithm design with faster convergence/better regret. right now it is implicit}
% Our algorithm for the one-step lookahead thresholding policy selects actions via a lower confidence bound (LCB) test against the threshold $\bm\gamma$. At each time $t$, the algorithm:
% \begin{enumerate}[leftmargin=12pt, itemsep=0pt, topsep=0pt, partopsep=0pt, parsep=1pt]
%     \item Constructs a candidate set $\tilde{G}_t$ of actions whose LCB for the one-step reward is at least $\gamma_t$.
%     \item If $|\tilde{G}_t| \geq 1$, it plays the action from this set with the highest LCB.
%     \item Otherwise, it plays an action uniformly at random.
% \end{enumerate}
After observing the reward $r_t$ and next state $s_{t+1}$, the algorithm updates the LCB for the observed state-action pair $(s,a)$. 
% We present our learning algorithm for the one-step lookahead 
% which builds on using lower confidence bound to check whether the reward is above the threshold $\bm\gamma$. At each time $t$, Algorithm~\ref{alg:'name'} proceeds as follows: First, the algorithm obtains a set of actions, denoted by $\tilde{G}_t$ whose estimated lower confidence bound of the reward is above the threshold $\gamma_t$. If the cardinality of $\tilde{G}_t$ is larger than one, then we play an action that is in $\tilde{G}_t$ according to the order of lower confidence bound. Otherwise, we will uniformly play a random action.
% We then observe the corresponding reward $r_t$ and the next state $s_{t+1}$. Next we update the lower confidence bound $\mathrm{LCB}^{(t+1)}$.
Let $\hat{r}_{s,a}^{\mathrm{1}}{(t+1)}$ be the empirical mean reward and $N_{s,a}^{(t+1)}$ be the visit count. 
% Let $\hat\phi(s,a)$ be the cumulative reward of playing action $a$ under state $s$ and $N_{s,a}^{(t+1)}$ be the number of times playing action $a$ under state $s$. 
The $\mathrm{LCB}_{s,a}^{(t+1)}$ is defined by $\hat{r}_{s,a}^\mathrm{1} - { \sqrt{\frac{g\big(N_{s, a} + 2\big)}{N_{s, a} + 2}}}$,
% as 
% $\hat{r}_{s,a}^\mathrm{1}{(t+1)} - \sqrt{\frac{g\left(N_{s, a}^{(t+1)} + 2\right)}{N_{s, a}^{(t+1)} + 2}}$,
with $g(t)=3\log(t)$. 

% This design, which uses the LCB (rather than the empirical mean) to define
A key distinction from prior work (e.g., Algorithm 1 of \citet{michel2022regret}) is our use of the LCB—rather than the empirical mean—to define
$\tilde{G}_t$.
% is a key distinction from prior work (e.g., Algorithm 1 of \citet{michel2022regret}).
Our choice of $g(t)=3\log(t)$ is smaller than the typical anytime-valid LCB
% choice
of
$g(t)=\Theta(\log(t))$. This ensures that the LCBs of truly sub-threshold actions drop below $\gamma_t$ rapidly, enabling faster identification and elimination of suboptimal actions. 
This is crucial for obtaining our improved regret bound in Theorem~\ref{th:one-step-regret}.
Before stating our theorem, we first introduce the following assumption on the threshold $\bm\gamma$.
% Note that this lower confidence bound
% is only a function of $N_{s,a}^{(t)}$, the number of times 
% choosing action $a$ in
% state $s$. Compared to the typical
% anytime-valid LCB choice of
% $g(t)=\Theta(\log(t))$, our choice of $g$ is smaller, ensuring faster threshold crossing for actions with $r^\one$ above the threshold.
% We note that our algorithm shares similarities with Algorithm 1 in \citet{michel2022regret}, with a key distinction being the use of the LCB instead of the empirical mean when defining the candidate set $\tilde{G}_t$. This modification is crucial for achieving a superior regret bound compared to \citet{michel2022regret}. Intuitively, our choice of $g$ 
% ensures that the LCB of actions whose one-step lookahead reward is smaller than the threshold will fall below the threshold more quickly. This design intuitively allows for faster elimination of ``bad" actions, as demonstrated in the proof.
\begin{algorithm}[tb]
  \caption{\textbf{LG1T}: LCB-Guided 1-Step Thresholding}
    \label{alg:'name'}
    \begin{algorithmic}[1]
        \STATE{\bfseries Input:}  initial state $s$, exploration function $g(t)$, 
        % threshold
        $\bm\gamma$\;
        \STATE Initialize: $s_0=s$, $\forall s\in\mathcal{S},a\in\mathcal{A}: N_{s,a}^{(0)}=0$, $\hat{\phi}(s,a)=0$,
        $\mathrm{LCB}_{s,a}^{(0)}=-\infty$
        \FOR{$t=0, 1,\cdots,T$}
        % \STATE Observe $\bm s_t$ and receive corresponding $\bm r_t$
        \STATE Retrieve $ s_t$, and obtain the set of good actions $\tilde G_t = \left\{a \colon \mathrm{LCB}_{s_t,a}^{(t)}\geq \gamma_t\right\}$. Set $C=\tilde G_t$;
        % for all $t$, any $\bm{s}$\;
        % \STATE  Give actions to all agents in $\tilde G_t$, and set $a_{t+1}^i = 1 \;\forall i\in \tilde G_t$
        \STATE Select $\max(0, 1-|\tilde G_t|)$ action
        % agents 
        uniformly from $\mathcal{A}\setminus \tilde G_t$, and append them to $C$\;
            \label{line:uniform}
        % Pick random $C \subseteq \{1\ldots M\}$ of size $B$ without replacement, prioritizing agents from $L_{t-1}(\bm{s}_t)$
            \STATE Select an action from $C$ according to $\mathrm{LCB}^{(t)}$, and denote it by $a$. Set $a_{t} = a$\; 
            \STATE Observe $(r_t,s_{t+1})$\;
            % \STATE Observe $(\bm{s}_{t+1}, r^m_t)$ \label{alg:observe-next-state}\;
            % \Comment*[r]{Update the point estimate of state-action reward}
            % \Comment*[r]{Update the sample mean of incremental reward}
           \STATE $\hat{\phi}(s,a),N_{s,a}^{(t+1)},\mathrm{LCB}_{s,a}^{(t+1)}=$Update-LCB-1 ($\hat{\phi}(s,a),N_{s,a}^{(t)},r_t,s_t,a_t$) (Algorithm~\ref{alg:update_lcb_1})\;
        %         N_{s, a}^{(t+1)} = 
        % N_{s, a}^{(t)} + \mathbb{I}
        % \left\{a_{t} = a,s_t=s\right\}
        % $\;
        % $N_{s_t^m, m}(t+1) = 
        % N_{s_t^m, m}(t) + \mathbb{I}
        % \left\{a_{t+1}^m = 1\right\}$\;
    %     \STATE  \hspace{1em} 
    % $\hat{r}_{s,a}^\mathrm{1}{(t+1)}={\hat{\phi}(s, a)}/{N_{s, a}^{(t+1)}}$\label{alg:update-sample-mean}\;
        % \STATE  \hspace{1em} $\mathrm{LCB}^{(t+1)}_{s,a} = \hat{r}_{s,a}^\mathrm{1}{(t+1)} - \sqrt{\frac{g\left(N_{s, a}^{(t+1)} + 2\right)}{N_{s, a}^{(t+1)} + 2}}$\;
        \ENDFOR\;
    \end{algorithmic}
\end{algorithm}
% Next, we will theoretically show that Algorithm~\ref{alg:'name'} has a minimax optimal constant regret: $\mathcal{R}^{\bm\pi^{\one,\bm\gamma}}=\mathcal{O}(SA/\Delta_+^\one)$ where $\Delta_+^\one:=\min_{\Delta_{s,a,t}^\one\leq 0} |\Delta_{s,a,t}^\one|$ denotes the smallest gap for all actions that are above the threshold. 
% Before we formally state the theorem, we will introduce the following assumption on the threshold $\bm\gamma$.
\begin{assumption}[Exists A Good Action]\label{assumption:exists-a-good-action}
    Assume $\forall t\leq T$, $\forall s\in\mathcal{S}$, there exists an action $a\in\mathcal{A}$ such that $r_{s,a}^\steps\geq \gamma_t$.
\end{assumption}
Assumption~\ref{assumption:exists-a-good-action} is independent of the structure of the MDP, and can be 
% easily
satisfied by choosing a low $\bm\gamma$.
% . We will further discuss this assumption in Remark~\ref{remark:gamma-choice}.
% This assumption is solely on the choice of threshold and is not related with the structure of the MDP.

% \kyra{here}

% Now we are ready to present 
% Let $\Delta_{s,a,t}^\mathrm{1}$ denote $\gamma_t-r_{s,a}^\mathrm{1}$ and $\Delta^\steps$ denote $\min_{s,a,t}|\Delta_{s,a,t}|$.
% Let $\Delta_+^\mathrm{1}:=\min_{\Delta_{s,a,t}^\mathrm{1}\leq 0} |\Delta_{s,a,t}^\mathrm{1}|$ denote the smallest gap for all actions that are above the threshold. 
We will let $\Delta_{s,a,t}^\steps =\gamma_t-r_{s,a}^\steps$ denote the gap between the threshold and the 
K-step lookahead reward, and define $\Delta_K=\min_{s,a,t}|\Delta^K_{s,a,t}|$. Further, we define $\Delta_{+}^\steps:=\min_{\Delta_{s,a,t}^\steps\leq 0}|\Delta_{s,a,t}^\steps|$ as the smallest gap among actions whose one-step reward is above the threshold. 
Thm.~\ref{th:one-step-regret} establishes that
% which characterize the regret. Next, we will theoretically show that
Algorithm~\ref{alg:'name'} has a 
% minimax optimal constant 
regret of 
% $\mathcal{R}^{\bm\pi^{\one,\bm\gamma}}=
$\mathcal{O}(SA/\Delta_+^\mathrm{1})$: 
% where $\Delta_+^\one:=\min_{\Delta_{s,a,t}^\one\leq 0} |\Delta_{s,a,t}^\one|$ denotes the smallest gap for all actions that are above the threshold. 
\begin{theorem}\label{th:one-step-regret}
Under Assumption~\ref{assumption:exists-a-good-action} with $K=1$, the regret $\mathcal{R}^{\bm\pi^{\mathrm{1},\bm\gamma}}$ of Algorithm~\ref{alg:'name'} satisfies
% \vspace{-10pt}
    $
\mathcal{R}^{\bm\pi^{\mathrm{1},\bm\gamma}}=\min\left\{\mathcal{O}\left({SA}/{\Delta_+^\mathrm{1}}\right),\mathcal{O}\left(\sqrt{SAT}\right)\right\}.
    $
\end{theorem}
 \vspace{-5pt}
Thm.~\ref{th:one-step-regret} aligns with the minimax lower bound established in \citet{feng2025satisficingregretminimizationbandits} (Thm. 3). 
When $S=1$, Thm.~\ref{th:one-step-regret}
yields a regret of
% shows that our algorithm achieves a regret upper bound that is 
$\mathcal{O}\left(A/\Delta_{+}^\mathrm{1}\right)$, which depends only on the gap between the one-step lookahead reward above the threshold and the threshold.
This improves the regret bound of $\mathcal{O}\left(\max_{\Delta_{a,t}^\mathrm{1}>0}1/\Delta_{a,t}^\mathrm{1}\right)$ from \citet{michel2022regret}, which depends on the gap between rewards below the threshold and the threshold—typically smaller than $\Delta_{+}^\steps$. 
% This is achieved because of the novel design of LCB.
The proof  structure  of Thm.~\ref{th:one-step-regret} (in Appendix~\ref{appendix:proof-of-one-step-theorem}) is similar to that of a standard 
% structure of proofs in 
contextual bandit, but shows that our LCB design eliminates suboptimal actions in constant time.
\subsection{Online K-Step Lookahead}\label{subsec:online-k-threshold-algo}
% \vspace{-5pt}
% Building on this foundation, w
We now extend our approach to the general K-step lookahead thresholding policy, which provides greater flexibility for long-term planning.
Learning the K-step lookahead thresholding policy is hard because 
the K-step reward $r^\steps_{s_t,a_t}$ is not directly observed. We decompose it
% the K-step reward 
into two terms: $r_{s_t,a_t}^\mathrm{1}$ and $\E\left[r_{s_{t+1},a_{t+1,\step}^*}^\step|s_t,a_t\right]$. To estimate the K$-1$-step reward, the agent must, after taking action $a_t$ in state $s_t$, follow the optimal K-step continuation $a_{s_{t+1},\step}^*$ from the next state $s_{t+1}$. However, playing $a_{s_{t+1},\step}^*$ may itself yield low reward and increase regret. This creates a tension between minimizing regret and collecting the information needed for accurate long‑horizon estimation.
% it involves estimating the K-step reward $r^\steps_{s_t,a_t}$ at each time $t$ while this reward is not observable at time $t$. In order to collect samples of $r_{s,a}^\steps=r_{s,a}^\mathrm{1}+\E_{s'\sim P_a(s)}\left[r_{s',a_{s',\step}^*}^\step\right]$, we need to play $a_{s',\step}^*$ after playing action $a$ in state $s$. However, since $r_{s',a_{s',\step}}^\steps$ may be smaller than the threshold, playing such action will increase the regret. This tension between regret minimization and reward maximization defines one of the core difficulties of this problem. 

To mitigate this, we use an $\epsilon$-greedy framework. At each time $t$, with probability $\epsilon_t$ that inversely proportional to the number of times the learner spent in the previous state playing previous action $N_{s_{t-1},a_{t-1}}^{(t)}$, Algorithm~\ref{alg:k-step} runs Estimate-$r^\step$ (Algorithm~\ref{alg:subroutine-estimate-k-1}), which uses a \emph{given} algorithm $\mathrm{ALG}_{\step}(\cdot\mid s_{t-1},a_{t-1})$ to choose actions for the next $K-1$ consecutive steps and collect samples $\sum_{h=t}^{t+K-1}r_h$ to estimate $r_{s_{t},a_{s_t,\step}^*}^\step$. $\mathrm{ALG}_\step$ takes the state-action pair $(s_{t-1},a_{t-1})$ as input and performs its own internal updates
using collected data.
% to its estimators using the subsequent data it collects.
These updates are maintained independently from the main algorithm’s estimators. The choices of $\mathrm{ALG}_\step$ are discussed in Appendix~\ref{appendix:discussion-low-regret-algorithm}. With probability $1-\epsilon_t$, Algorithm~\ref{alg:k-step} uses the same procedure as described in Algorithm~\ref{alg:'name'} which uses LCB to determine whether the K-step lookahead reward is above the threshold. After playing the action, lower confidence bound is updated (Algorithm~\ref{alg:update_lcb_k}).

Theoretical regret depends on the error of samples  that is collected by $\mathrm{ALG}_\step(\cdot|s_{t-1},a_{t-1})$ to estimate $\mathrm{K-1}$-step lookahead reward $r_{s_t,a_{s_t,\step}^*}^\step$. Formally, we make the following assumptions to bound the error.
% change this to reduction to estimating the  lookahead reward\textbf{Reduction to Contextual Bandits}\;\; Further,
% identifying
% % selecting the correct action 
% $a_{s',\step}^*$ is itself nontrivial, as it requires estimating $r_{s',a'}^{\step}$, which is also not directly observed. 
% \emph{However, 
% % we argue that 
% this subproblem 
% % of action selection 
% is more tractable than the full K‑step reward estimation.}
% % estimating the full K-step reward.
% Given a preceding state-action pair $(s, a)$, the distribution of the next state $s'$ follows a fixed distribution, and selecting the subsequent $(K-1)$-step action sequence reduces to a \emph{contextual bandit} problem: the context is $s'$, the actions are all $(K-1)$-step sequences in $|\mathcal{A}|^{K-1}$ and the reward is the cumulative reward over those steps. We denote the regret of this auxiliary problem by $\mathcal{R}_{\text{contextual}}^{\step}$.

% % is fixed by the MDP dynamics. Therefore, choosing $a_{s',\step}^*$ reduces to a contextual bandit problem with context space $\mathcal{S}$ and an action space of size $|\mathcal{A}|^K$. We denote the regret for this contextual bandit problem by $\mathcal{R}_{\text{contextual}}^{\step}$. Based on this structural simplification, we proceed under the assumption that a low-regret algorithm exists for this specific contextual bandit task. 
% We therefore proceed under the following assumption, which assumes the existence of a low‑regret algorithm for this contextual bandit task.

% hint in practice this is not needed, just for applying concentrtaion/proof
\vspace{-2pt}
\begin{assumption}[Exists A Good Sampling Policy]\label{assumption:low-regret-algorithm}
Given any state-action pair $(s,a)$ and a positive constant $H$, 
run 
% suppose 
$\mathrm{ALG}_{\step}(\cdot|s,a)$
% is run 
for $H$ times and collect $H$ trajectories $\left\{(s_0^h,a_0^h,s_1^h,a_1^h,\cdots,s_{K-2}^h,a_{K-2}^h)\right\}_{h\in[H]}$ 
% for $h\in[H]$ 
where $s_0^h\sim P_a(s)$. Then with probability $1-\delta$:
$
   \E_{s'\sim P_{a}(s)}\left[r_{s',a_{s',\step}^*}^\step\right]-\sum_{h=1}^H\sum_{k=0}^{K-2}R_{s_k^h,a_k^h}=\mathcal{O}\left(\sqrt{C_{K-1}H\log(SAH/\delta)}\right).$
\end{assumption}
We emphasize that Assumption~\ref{assumption:low-regret-algorithm} can be satisfied by existing algorithms. Details are deferred to Appendix~\ref{appendix:discussion-low-regret-algorithm}. Next, we present the regret of Algorithm~\ref{alg:k-step}.
% Under Assumption~\ref{assumption:low-regret-algorithm}, we can estimate $r_{s,a}^\steps$ by estimating $r_{s,a}^\mathrm{1}$ and $r_{s',a_{s',\step}^*}^\step$ separately. And to estimate $r_{s',a_{s',\step}^*}^\step$, we can run $\mathrm{ALG}_{K-1}(\cdot|s,a)$ for the next $K-1$ steps. To balance the increase of regret by taking $K-1$ steps potentially ``bad'' actions as mentioned in the beginning of Section~\ref{subsec:online-k-threshold-algo}, we use a $\epsilon$ greedy framework. Formally, we will present Algorithm~\ref{alg:k-step}: At each time $t$, with probability $\epsilon_t$ that inversely proportional to the number of times the learner spent in the previous state playing previous action, we will run $\mathrm{ALG}_{\step}$ for $K-1$ consecutive steps and collect samples for $r_{s_{t-1},a_{t-1}}^\step$ via Estimate-$r^\step$ (Algorithm~\ref{alg:subroutine-estimate-k-1}). $\mathrm{ALG}_{\step}$ only takes the state-action pair $s_{t-1},a_{t-1}$ as input and performs its own internal updates to its estimators using the subsequent data it collects. These updates are maintained independently from the main algorithm’s estimators. Otherwise, we will
\begin{algorithm*}[!t]
  \caption{\textbf{LGKT}: LCB-Guided K-Step Thresholding}
    \label{alg:k-step}
    \begin{algorithmic}[1]
        \STATE{\bfseries Input:}  initial state $s$, exploration function $g(t)$, threshold $\bm\gamma$, instance parameter $\eta$, exploration decay rate $p$, sampling algorithm that takes in initial state-action pair: $\mathrm{ALG}_{\mathrm{K}}(\cdot|\text{initial state,initial action})$\;
        \STATE Initialize: $s_0=s$, $\forall (s,a)\in\mathcal{S}\times\mathcal{A}: N_{s,a}^{(0)}= N_{s,a,\step}^{(0)}=\hat{\phi}_{s,a}^{\mathrm{1}}=\hat\phi_{s,a}^\step=0$,
        $\mathrm{LCB}_{s,a}^{(0)}=-\infty$
        \WHILE{$t\leq T$}
        % \STATE Observe $\bm s_t$ and receive corresponding $\bm r_t$
        \STATE Retrieve $ s_t$, and obtain the set of good actions $\tilde G_t = \left\{a \colon \mathrm{LCB}_{s_t,a}^{(t)}\geq \gamma_t\right\}$. Set $C=\tilde G_t$;
        % for all $t$, any $\bm{s}$\;
        % \STATE  Give actions to all agents in $\tilde G_t$, and set $a_{t+1}^i = 1 \;\forall i\in \tilde G_t$
        \STATE Select $\max(0, 1-|\tilde G_t|)$ actions uniformly from $\mathcal{A}\setminus \tilde G_t$, and append them to $C$
        % Pick random $C \subseteq \{1\ldots M\}$ of size $B$ without replacement, prioritizing agents from $L_{t-1}(\bm{s}_t)$
        \label{alg:check-lcb-two}\;
            \STATE Select an action from $C$ according to the order of $\mathrm{LCB}^{(t)}$, and denote it by $a$. \;
         \IF{ With probability $\epsilon_t=\min\left\{1,\frac{1}{\left(N_{s_{t-1},a_{t-1}}^{(t)}+1\right)^p\min\{\eta,1/2\}}\right\}$:}
         \STATE $N_{s,a,\step}^{(t+K-1)},N_{s,a}^{(t+K-1)},\hat{\phi}^\step,\hat{\phi}^\mathrm{1},\mathrm{LCB}^{(t+K-1)}=$\; Estimate-$r^\step(\mathrm{ALG}_\step,s_{t-1},a_{t-1},s_t,t)$ (Algorithm~\ref{alg:subroutine-estimate-k-1})\;
         \STATE $t=t+K-1$\;
        %\STATE $\quad$ For all $m \in \tilde U$: $a_{t+1}^m \gets  1$; for all $m \notin \tilde U$: $a_{t+1}^m \gets  0$.
       
       \ELSE{}
       \STATE K-Step thresholding: set $a_{t} = a$\;
       and observe $(r_t,s_{t+1})$\;
       \STATE $\hat{\phi}^\mathrm{1}_{s,a},N_{s,a}^{(t+1)},N_{s,a,\step}^{(t+1)},\mathrm{LCB}_{s,a}^{(t+1)}=$Update-LCB-K($\hat\phi^1_{s,a},\hat{\phi}^\step_{s,a},N_{s,a}^{(t)},N_{s,a,\step}^{(t)},r_t,s_t,a_t$) (Algorithm~\ref{alg:update_lcb_k})\;
        \STATE $t=t+1$\;
       \ENDIF
        \ENDWHILE\;
    \end{algorithmic}
\end{algorithm*}
\begin{theorem}\label{th:k-step-regret}
Under Assumptions~\ref{assumption:exists-a-good-action}, \ref{assumption:low-regret-algorithm}, let $\eta=\Delta^\steps,p=1/2,g(t)=\sqrt{\frac{3\log(t)}{t}}$, Algorithm~\ref{alg:k-step} admits the following regret for any K:
    $\mathcal{R}^{\bm\pi^{\steps,\bm\gamma}}=\max\left\{\mathcal{O}\left((K-1)\sqrt{SAT}\right),\mathcal{O}\left(C_{K-1}\sqrt{SAT\log(T)}\right)\right\}$.
\end{theorem}
% \kyra{place in the right place:Let $\Delta_{s,a,t}^\steps$ denote $\gamma_t-r_{s,a}^\steps$ and $\Delta^\steps$ denote $\min_{s,a,t}|\Delta_{s,a,t}|$.}
% As discussed in Appendix~\ref{appendix:discussion-low-regret-algorithm}, 
When $K=2$, UCB algorithm satisfies Assumption~\ref{assumption:low-regret-algorithm} with $C_{K-1}=A$ (Appendix~\ref{appendix:discussion-low-regret-algorithm}). Therefore, we immediately have the following corollary.
\begin{corollary}
    For 2-step lookahead thresholding, let $\mathrm{ALG_1}$ be the UCB algorithm, $\eta=\Delta^\mathrm{2}/\sqrt{A},p=1/2,g(t)=\sqrt{\frac{3\log(t)}{t}}$, Algorithm~\ref{alg:k-step} has the following regret: $
\mathcal{R}^{\bm\pi^{\mathrm{2}},\bm\gamma}=\mathcal{O}\left(A\sqrt{ST\log(T)}\right)$.
\end{corollary}
% Now we discuss the effect of $K$ on the convergence rate shown in Thm.~\ref{th:k-step-regret}.
\begin{remark}[Impact of K On Convergence Rate]\label{remark:impact-of-k-regret}
      Thm.~\ref{th:k-step-regret} demonstrates that when the number of steps lookahead increases, the convergence will be slower. Further, when $\max({K,C_{K-1}}\sqrt{SA})=o(\sqrt{T})$, Thm.~\ref{th:k-step-regret} enjoys a sublinear convergence rate to $\bm{\pi^{\steps,\bm\gamma}}$, improving over the linear convergence rate to the optimal policy \citep{10.5555/3305381.3305409}. This dependence on $K$ shows there is a trade-off between maximizing the convergence rate and maximizing the cumulative reward. Empirically, we show that $K=1,2$ can already outperforms SOTA tabular RL algorithms for a long horizon for all instances we test (Section~\ref{sec:experiments}). 
\end{remark}
% Next, we provide a proof sketch of Thm.~\ref{th:k-step-regret}. The procedure
The proof of Thm.~\ref{th:k-step-regret} (Appendix~\ref{appendix:proof-of-k-step}) 
follows the structure of Thm.~\ref{th:one-step-regret}, with two key modifications: 1) Bounding \emph{regret from exploration}, 
2) Ensuring \emph{sufficient estimation samples}. A proof sketch is provided in Appendix~\ref{appendix:proof-of-k-step}.
% : the same choice of
% % We show that such choice of
% $\epsilon_t$ ensures that for each state-action pair $s,a$, at least $\frac{K}{2}\sqrt{SAN_{s,a}^{(T)}}$ times picking actions to estimate $r_{s,a}^\steps$.
% Recall that the
% % As we mentioned before Thm.~\ref{th:k-step-regret}, we estimate
% K-step lookahead reward comprises
% % by separating it into 
% a 1-step lookahead reward and a K$-1$-step lookahead reward.
% Accurate estimation of the latter
% % Estimating 
% % K$-1$-step lookahead reward 
% is ensured by
% % can be achieved by using $\mathrm{ALG}_{\step}$ in
% Assumption~\ref{assumption:low-regret-algorithm}. 
% % Therefore, w
% We show that $\frac{K}{2}\sqrt{SAN_{s,a}^{(T)}}$ runs of
% % times of running 
% $\mathrm{ALG}_\step$ suffices to obtain
% % is enough to guarantee
% an accurate estimation of $r^\steps$. 
% The details are deferred to Appendix~\ref{appendix:proof-of-k-step}.

% Before we end this section, we will discuss on how to choose the threshold $\bm\gamma$.
Selecting the threshold 
$\bm\gamma$ involves a fundamental trade-off. A low threshold ensures Assumption~\ref{assumption:exists-a-good-action} and fast convergence (Theorem~\ref{th:one-step-regret}), but potentially selects from more actions,  lowering cumulative reward. 
A higher threshold yields slower convergence but can converge to a better policy. Thus, we must balance between maximizing cumulative reward and maximizing convergence rate.
    % There is a fundamental trade-off between selecting a threshold $\bm\gamma$ that satisfies Assumption~\ref{assumption:exists-a-good-action} and maximizing the cumulative reward. Though a low threshold can ensure fast convergence and as shown in Theorem~\ref{th:one-step-regret}, higher threshold will lead to slower convergence, a low threshold will also make more actions that can be chosen. This potentially will lower the cumulative reward. Our objective is still maximizing the cumulative reward through increasing the convergence rate. Therefore, we need to balance between maximizing the reward and maximizing the convergence rate. 
    % Empirically (Appendix~\ref{appendix:ablation-threshold}), 
    Theoretically, a low $\bm\gamma$ may perform better early, while a high $\bm\gamma$ can eventually achieve higher reward. Nevertheless, as shown in Section~\ref{sec:experiments}, moderate values of $\bm\gamma$ achieve strong performance, with ablation studies (Appendix~\ref{appendix:ablation-threshold}) indicating that results are often robust to this choice across many environments.

    Further, Theorems~\ref{th:one-step-regret} and \ref{th:k-step-regret} hold for time and state varying threshold. Therefore, one can start from a low threshold and increase the threshold gradually to achieve both fast convergence and asymptotic optimality. We provide one heuristic of doing so in Appendix~\ref{appendix:adaptive-threshold}.
% We discuss the choice of $\bm\gamma$ in Appendix~\ref{remark:gamma-choice}. In Sec.~\ref{sec:experiments}, we observe that moderate values of $\bm\gamma$ achieve strong performance. Ablation studies (Appendix~\ref{appendix:ablation-threshold}) indicate that results are often robust to this choice across many environments.

Before we show the empirical results, we discuss the regime where fast convergence translates to higher cumulative reward.
\begin{remark}[Favorable Regime of Fast Convergence]
    As Theorem~\ref{th:k-step-regret} shows, LGKT achieves sublinear regret compared with K-step lookahead thresholding policy. Since the regret lower bound compared with the exact optimal policy is linear (Section~\ref{subsec:hardness}), in order to theoretically characterize the favorable regime of fast convergence, we only need to compare the regret lower bound of tracing exact optimal policy with the suboptimality gap of K-step lookahead thresholding policy as characterized in Theorem~\ref{th:instance-lower-bound}.
\end{remark}

% \kyra{move to appendix and fix link}
% \vspace{-5pt}
\section{Experiments}\label{sec:experiments}
% \vspace{-5pt}
We evaluate our proposed algorithms—\textbf{LG1T} (Algorithm~\ref{alg:'name'}), \textbf{LG2T} (Algorithm~\ref{alg:k-step}), and an adaptive variant \textbf{LG1‑2T} (Algorithm~\ref{alg:lg1-2t})—across several environments. LG1‑2T gradually increases the lookahead steps during learning, aiming to balance fast initial convergence with long‑term reward maximization. Because standard benchmarks for non-episodic finite-horizon settings are scarce, we test our methods in both non-episodic and episodic regimes.
Our non‑episodic experiments include:
% We numerically test the performance of our algorithms, showing that LG1T (Algorithm~\ref{alg:'name'}) and LG2T (Algorithm~\ref{alg:k-step}) consistently achieve the highest cumulative average reward among non-episodic environments with discrete state space: 
(1) 2000 \textbf{synthetic MDPs} with discrete states, 
(2) \textbf{JumpRiversSwim} \citep{wei2020model} with three different sets of state spaces,
(3) \textbf{FrozenLake} \citep{brockman2016openaigym} ($4\times 4$ map). 
We additionally test on an episodic environment with continuous state space, (4) \textbf{AnyTrading} \citep{gym_anytrading}. 
We describe the environment details in Appendix~\ref{paragraph:synthetic}.
Across all experiments, we set the total horizon length to be $20,000$.
We set the threshold for LG1T to be 0.3 and LG2T to be 0.9. 
The LG1‑2T algorithm switches from LG1T to LG2T at a specified change time, and we set the change time to 100 for Experiments (1) and (2), 10,000 for Experiment (3), and 30 for Experiment (4).
% We set the change time of LG1-2T to be 100 for experiment 1) and 2). For experiment 3), we set the change time to be 10000 and for experiment 4) we set the change time to be 30. The change time corresponds to the time changing from LG1T to LG2T. 
An ablation on the choice of threshold is included in Appendix~\ref{appendix:ablation-threshold}. We provide a heuristic for adaptively tuning the threshold in Appendix~\ref{appendix:adaptive-threshold} and a heuristic for adaptively choosing the change time for LG1-2T in Appendix~\ref{appendix:adaptive-k}. 
% \kyra{common experimental setups should be described here. horizon length, etc. should also discuss choice of threshold here. and think about whether anything else is common, link to ablation on threshold choice}
Our methods consistently achieve the highest cumulative reward across these tests.

\textbf{Benchmarks\;\;} We compare against \textbf{six} state-of-the-art \textbf{tabular RL algorithms}. Three model-based average-reward algorithms: 1) UCRL2 \citep{NIPS2008_e4a6222c}, 2) KLUCRL \citep{filippi2010optimism}, 3) PMEVI-KLUCRL \citep{boone2024achieving}; three model-free average-reward algorithms: 4) MDP-OOMD \citep{wei2020model}, 5) Optimistic Q Learning \citep{wei2020model} and episodic RL: 6) Q Learning \citep{jin2018q}. 
Hyperparameters for all benchmarks are set as reported in their original papers. 
We also evaluate \textbf{two oracle policies}: 1) 1-step lookahead greedy $\bm\pi^{\mathrm{1},\greedy}$ and 2) 2-step lookahead greedy $\bm\pi^{\mathrm{2},\greedy}$, and compute the optimality ratio of each policy relative to the full oracle policy.
Descriptions of the benchmark algorithms along with implementation details are provided in Appendix~\ref{appendix:implementation}.

% We choose the hyper-parameters of the benchmarks exactly the same as reported in their papers. The implementation details and discussion on the benchmarks are deferred to Appendix~\ref{appendix:implementation}.
% \kyra{explanation of what these algorithms in the appendix and link. You may want to justify why these algorithms if not in the main, then in the appendix. Maybe just state shortly that they are state-of-art whose implementation is easy to find/replicate}
% Further, we benchmark our algorithm with two \textbf{oracle policies}: 1) 1-step lookahead greedy $\bm\pi^{\mathrm{1},\greedy}$ and 2) 2-step lookahead greedy $\bm\pi^{\mathrm{2},\greedy}$. 
% We compute the optimality ratio of each policy relative to the full oracle policy.
% We also compute the optimality ratio between the above two oracle policies with the oracle policy. The details of the implementation can be found in Appendix~\ref{appendix:implementation}.

\textbf{Modifications\;\;} In all experiments, we modify LG1T and LG2T by replacing the uniform random selection (Line 5) with a UCB-based choice:
% ,  we choose 
actions are chosen in decreasing order of the index
% a UCB index defined as
% by changing one line to maximize the reward. Specifically, instead of uniformly choose an action (line~5), we will choose according to the order of UCB defined as 
$\hat{r}_{s,a}(t)+\frac{3.4}{N_{s,a}^{(t)}}\sqrt{\frac{\log\log\left(N_{s,a}^{(t)}\right)+\log(10T)}{N_{s,a}^{(t)}}}$.
% which shrinks the any time valid confidence bound
This index uses an anytime‑valid confidence bound \citep{howard2021time} scaled by $1/N_{s,a}^{(t)}$, a heuristic scaling 
that discourages excessive exploration on poorly performing actions and accelerates convergence \citep{jin2018q}. We show in Appendix~\ref{appendix:theory-version} that though this modification can further improve the cumulative reward, the unmodified version still outperforms all benchmarks with a low threshold. This shows that the empirical gain comes from the thresholding idea and the use of LCB to test whether the reward is above threshold. Using UCB instead of uniform when no actions' LCB is above threshold mainly improves robustness and performance when the threshold is high. We further show in Appendix~\ref{appendix:regret-modified} that the modified Algorithm~\ref{alg:'name'} achieves the same convergence rate up to log factor.
% aligns with \citet{jin2018q} which discourages exploration on bad actions to accelerate convergence.

% Further, we propose a variant of our algorithm that gradually increases the number of lookahead steps during learning, aiming to balance fast initial convergence with long‑term reward maximization (LG1‑2T, Algorithm~\ref{alg:lg1-2t}\kyra{link to appendix}).

% We will also test numerically a framework that will gradually increase the number of steps lookahead to balance between fast convergence and reward maximization (LG1-2T, Algorithm~\ref{alg:lg1-2t}).
% \vspace{-5pt}
\subsection{Discrete State Space}
% \vspace{-5pt}
\textbf{Synthetic MDP\;\;}
Figure~\ref{fig:results-synthetic} presents the results on 1,000 randomly generated MDPs with 10 states and 5 actions, and 1,000 with 100 states and 25 actions.
% % (see Appendix~\ref{paragraph:synthetic}).
% Our algorithms—LG1T, 
% % (with threshold 0.3),
% LG2T,
% % (with threshold 0.9), 
% and LG1-2T—consistently outperform all baselines over long 20000 horizons.
% Notably,
Our algorithms (LG1T, 
LG2T, 
and LG1-2T) converge rapidly to their respective oracle policies, achieving the highest final performance. For the 10-state case, the average competitive ratios across
% averaged over 
1000 instances are
% for $\bm\pi^{\mathrm{1},\greedy}$ and $\bm\pi^{\mathrm{2},\greedy}$ compared with the optimal policy is 
$V_0^{\bm\pi^{\mathrm{1},\greedy}}/V_0^*=0.75$ and $V_0^{\bm\pi^{\mathrm{2},\greedy}}/V_0^*=0.96$.
The strong empirical performance of our algorithms, even when learning suboptimal oracles, demonstrates that existing RL methods suffer from slow convergence, while our approach quickly converges to high‑quality policies.
% The strong performance of our algorithms even when the oracle policies we are learning is suboptimal demonstrates that existing RL algorithms suffer from slow convergence, preventing them from competing with our methods that quickly converge to high-quality suboptimal policies.

In the 100-state case, the average competitive ratios are $0.8$ for $\bm\pi^{\mathrm{1},\greedy}$ and $0.94$ for $\bm\pi^{\mathrm{2},\greedy}$.
PMEVI‑KLUCRL is excluded due to its prohibitive runtime ($>$3 h/instance); its performance is expected to be similar to KLUCRL \citep{boone2024achieving}.
While model‑based average‑reward algorithms initially achieve higher reward, both LG1T and LG1‑2T surpass them within a short horizon (Fig. \ref{fig:results-synthetic} right). This early advantage stems from their optimistic exploration, which avoids purely random actions, whereas our methods begin with a brief uniform exploration phase. After fewer than 1000 steps, our algorithms consistently outperform all others over horizons of 20000 steps—despite the suboptimality of the 1‑step and 2‑step oracles.
While $\bm\pi^{\mathrm{2},\greedy}$
yields higher rewards than
% is better than 
$\bm\pi^{\mathrm{1},\greedy}$, LG2T's slower convergence than LG1T (Thm.s~\ref{th:one-step-regret} and~\ref{th:k-step-regret}) leads to slow initial reward accumulation. 
This is evident in Fig.~\ref{fig:results-synthetic}, where LG1T outperforms LG2T over the early horizon (left) and over the full horizon in larger state spaces (right).
% When the state space is small, LG2T often catches up quickly, making a warm start less critical. However, 
As the state space grows, the convergence gap widens, and warm-starting becomes advantageous. LG1‑2T is designed to balance rapid early convergence with better long-term performance.
% , we introduce LG1‑2T, which initializes the 2‑step policy using converged 1‑step estimates. 
This hybrid achieves the best overall results in larger state spaces (Fig.~\ref{fig:results-synthetic}, right). We also observe that model-based methods outperform model-free methods which aligns with observations from prior studies \citep{wei2020model}.
% The performance gap widens as the state space grows, making a warm start essential. Our solution, LG1‑2T, warm-starts the 2‑step policy with the converged estimates from the 1‑step algorithm, thereby balancing LG1T’s rapid early convergence with LG2T’s superior long‑term reward. This hybrid yields the best overall performance (Fig.~\ref{fig:results-synthetic}, right).
% \kyra{our observation that model-based methods outperform model-free methods aligns with observations from prior studies, cite}

% \kyra{here}
% \kyra{same discussion missing for 10 state case}
% in Section~\ref{sec:online-algorithm}).
% And this slow convergence translates to a slow improvement of reward in the beginning. To mitigate this trade-off, we introduce a warm-start procedure, initializing the 2-step policy with the converged estimates from the 1-step algorithm (LG1-2T). This hybrid approach effectively balances the fast initial convergence of LG1T with the superior long-term performance of LG2T, yielding a policy that has the best performance.
\begin{figure}[t]
  % \vskip -0.15in
  \begin{center}
    \centerline{\includegraphics[width=\linewidth]{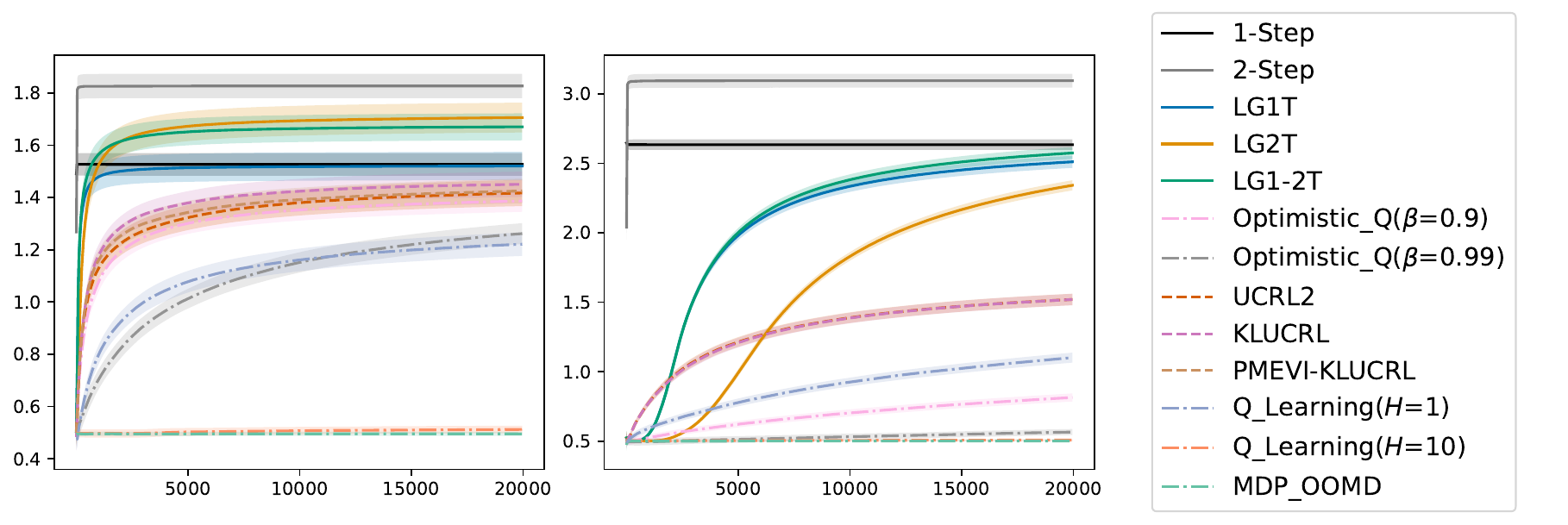}}
    \vspace{-5pt}
    \caption{
      Running average reward over 1000 distinct MDP instances. Left: $S=10,A=5$, Right: $S=100,A=25$. Right excludes PMEVI-KLUCRL due to its prohibitive runtime and on Right, KLUCRL overlaps with UCRL2.
       % \kyra{explain if methods are excluded in one of them}
    }
    \label{fig:results-synthetic}
  \end{center}
  % \vspace{-\baselineskip}
  \vskip -0.4in
\end{figure}

% \kyra{here}
\textbf{JumpRiverSwim 
% \citep{wei2020model}
\;\;}
We evaluate our algorithms on the JumpRiverSwim environment, a chain MDP with states $\{0,\cdots,S-1\}$, $(S=5,8,15)$, and actions $\{\texttt{left}, \texttt{right}\}$ where we discuss in detail in Appendix~\ref{paragraph:jumpriverswim}. In this setting, any K-step lookahead policy with short $K$ is suboptimal:
it will myopically choose \texttt{left} at state $0$, whereas only a sufficiently long-horizon lookahead discovers the optimal trajectory to the right. 
% for the state space we choose. This is because K-step lookahead policy with short time planning will choose left when starting from staet 0 and only end-of-horizon lookahead will be optimal.
Despite this,
% inherent suboptimality,
our algorithms (LG1T, LG2T, and LG1-2T) consistently achieve higher cumulative reward than all RL baselines across all state-space sizes over 20,000 steps (Fig.~\ref{fig:results-riverswim}). 

% This demonstrates the practical benefit of fast convergence even when the oracle policy is not globally optimal.
The performance gain stems from our method’s rapid identification of high-value actions at critical states. For example, LG1T quickly learns that $\texttt{right}$ is optimal at the rightmost state once that state is visited (which occurs early due to initial random exploration). In contrast, standard RL baselines must learn both the transition dynamics and the value function across the entire chain, leading to substantially slower convergence. 
Because no short-horizon (K-step) lookahead policy is optimal in this environment, we also implement a hybrid approach, LG1T‑RL (Algorithm~\ref{alg:lg1trl}), which switches from LG1T to the PMEVI‑KLUCRL algorithm after 10,000 steps. As shown in Figure~\ref{fig:results-riverswim}, LG1T‑RL consistently matches or outperforms the standalone PMEVI‑KLUCRL baseline and converges to identical behavior after the switch. This shows that our method can effectively warm‑start a standard RL algorithm, delivering stronger finite‑horizon performance during early learning while retaining asymptotic convergence to the optimal policy.
\begin{figure}[t]
  % \vskip -0.15in
  \begin{center}
    \centerline{\includegraphics[width=\columnwidth]{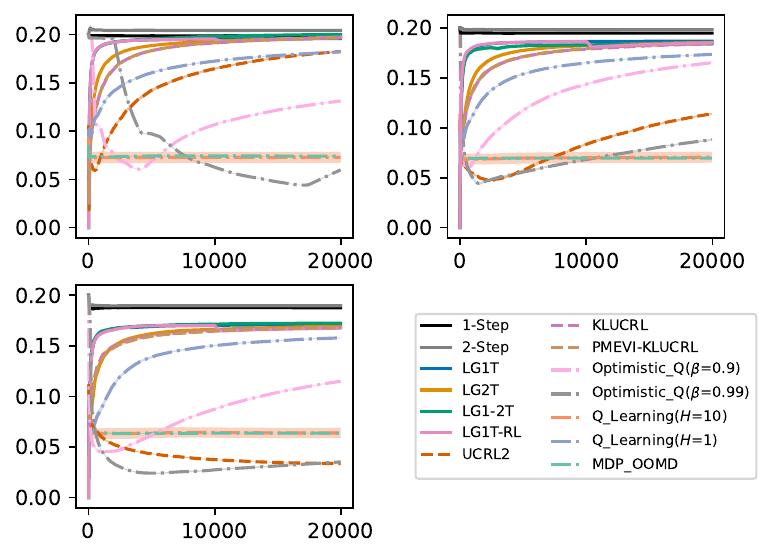}}
    \vspace{-3pt}
    \caption{
      Running average reward under JumpRiverSwim environments averaged over 100 repetitions. Top Left: 5 states, Top Right: 8 states, Bottom Left: 15 states.
    }
    \label{fig:results-riverswim}
  \end{center}
\vskip -0.4in
\end{figure}
% Since no short K-step lookahead policies are optimal, we implement a hybrid approach LG1T-RL (Algorithm~\ref{alg:lg1trl}) which will change from LG1T to PMEVI-KLUCRL after 10000 steps. As shown in Figure~\ref{fig:results-riverswim}, LG1T-RL consistently matches or exceeds the performance of the standalone PMEVI-KLUCRL baseline and converges to identical behavior after the switch. This demonstrates that our method can effectively warm-start a standard RL algorithm, providing superior finite-horizon performance during initial learning while preserving asymptotic convergence to the optimal policy.
% this algorithm is never worse than PMEVI-KLUCRL and will behave the same as PMEVI-KLUCRL after 10000. This shows that our proposed algorithm can serve as the warm start for RL methods to have better finite horizon performance while still converge to the optimal policy.

\textbf{FrozenLake 
% \citep{brockman2016openaigym}
\;\;}
We evaluate our approach on the FrozenLake environment, a $4 \times 4$ grid world ($S=16$) which we describe in detail in Appendix~\ref{paragraph:Frozenlake}. 
% The agent begins in the top-left corner and must navigate to the goal in the top-right corner while avoiding holes. We assign rewards as follows: $+1$ for reaching the goal, $0$ for falling into a hole, and $+0.2$ for all other transitions (see Appendix~\ref{paragraph:Frozenlake}).
We omit 
% exclude
oracle benchmarks as the underlying model
% because the model 
is not directly accessible. As in JumpRiverSwim, an optimal policy requires end-of-horizon planning; any finite $K$ step lookahead oracle is suboptimal. Nevertheless, Figure~\ref{fig:results-frozenlake-trading} (Left) shows that our algorithms (LG1T, LG2T) consistently outperform all benchmark RL algorithms. 
This advantage arises because standard RL methods converge slowly and frequently fall into hazardous states, whereas our approach rapidly identifies and exploits low‑risk trajectories, yielding higher cumulative reward.
% This empirical advantage stems from the slow convergence of standard RL methods, which fail to learn a safe policy within the given sample budget and frequently fall into holes. In contrast, our approach rapidly identifies and exploits low-risk paths, achieving higher cumulative reward.
% We tested under Frozen Lake environment where the state space is a $4\times4$ grid. The goal is to start from the left top corner and go to the right top corner while keeping away from the hole. Therefore, we set the reward to be 0 when reaching a hole, 1 when reaching the goal and 0.2 otherwise. Details are referred to Appendix~\ref{paragraph:Frozenlake}. Similar to JumpRiverSwim, though only end-of-horizon lookahead is optimal, LG1T,LG2T and LG1-2T consistently outperforms all benchmarks for a long 20000 horizon as shown in Figure~\ref{fig:results-frozenlake}. This is because RL algorithms converge very slowly and can not avoid hitting the hole within limited samples. 
% \vspace{-5pt}
\subsection{Continuous State Space: AnyTrading}
% \vspace{-5pt}
\begin{figure}[t]
  % \vskip -0.1in
  \begin{center}
    \centerline{\includegraphics[width=\columnwidth]{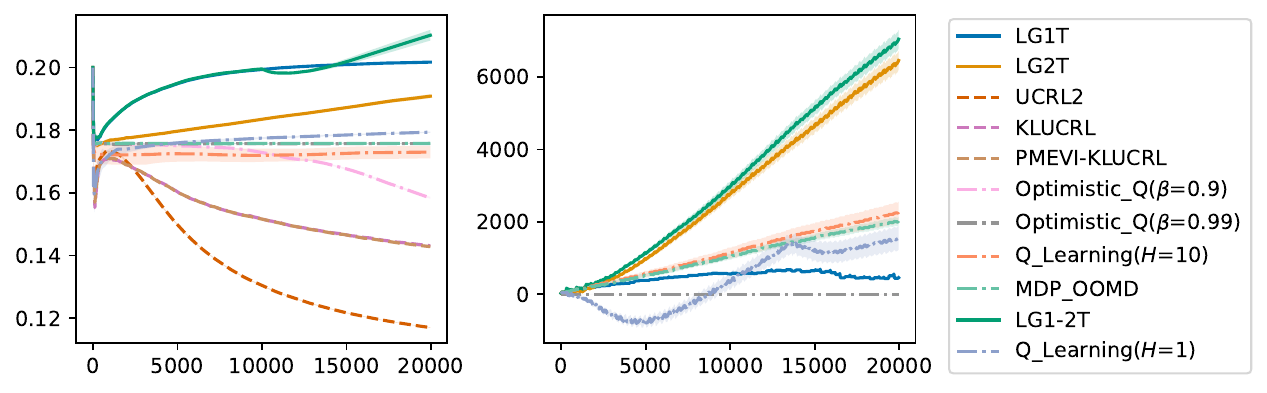}}
    \caption{
      Left: Running average reward under FrozenLake environments with $4\times 4$ state space averaged over 100 repetitions. LG1-2T has parameter $t_c=10000$, Right: Cumulative reward of 20000 environment steps under AnyTrading environment averaged over 1000 repetitions. LG1-2T has parameter $t_c=30$. Right exclude model based methods for prohibitive long runtime.
    }
    \label{fig:results-frozenlake-trading}
  \end{center}
  \vskip -0.4in
\end{figure}
To address the scarcity of non-episodic benchmarks, we evaluate our algorithms in the AnyTrading environment, which simulates trading on real market data with continuous state spaces (stock prices) and binary actions (buy/sell). See Appendix~\ref{paragraph:anytrading} for details. Given the computational cost of tabular model‑based methods in continuous domains, we exclude them here. Experiments use environment resets with a fixed budget of 20,000 agent‑environment interactions.
% Due to the lack of non-episodic environments,
% we evaluate our algorithms in continuous state spaces. In these domains, tabular model‑based methods are computationally prohibitive and not readily available, so we exclude them. 
% We exclude tabular model‑based methods due to their prohibitive computational cost and the lack of available implementations for continuous domains.
% Therefore, we allow the system to be reset and will fix the total number of agent-environment interventions to be 20000.
% AnyTrading is a simulated trading environment using real market data. The state space is the stock price and the action is binary: sell or buy. Details are referred to Appendix~\ref{paragraph:anytrading}. 

As shown in Figure~\ref{fig:results-frozenlake-trading} (Right), the suboptimality of the 1‑step lookahead policy causes LG1T to be surpassed after a short horizon. In contrast, LG2T outperforms all other baselines within the budget. Moreover, warm‑starting with LG1T in the LG1‑2T variant yields even better performance, demonstrating the benefit of gradually increasing the planning horizon.
A key insight from our experiments is that the fast convergence of our approach extends naturally from discrete to continuous state spaces, as demonstrated in the AnyTrading environment. This performance translates directly to superior results against benchmarks, even with raw state observations and no modification to the observation space.
However, an important limitation arises from the tabular foundation of our methods: like all such RL approaches, they require multiple visits to similar states to identify optimal actions. This necessitates a degree of state recurrence—an assumption that may not hold in all continuous environments where states are rarely revisited. Extending this work to deep learning methods is an important direction for future research.
\section{Conclusion}
We conclude by explaining the empirical success of our method. As shown in Section~\ref{subsec:hardness}, the linear lower bound on learning the optimal average‑reward policy makes standard RL inherently slow in the non‑episodic setting. Our approach circumvents this by targeting a K‑step lookahead thresholding policy, which enjoys a fast convergence rate.
Notably,  moving from a greedy to a thresholding rule is essential. While Q‑learning with 
$H=1$ (a UCB contextual bandit algorithm) also learns quickly in isolation, it can choose poor initial actions that trap the agent in low‑value states, from which recovery is slow. Our thresholding mechanism uses lower confidence bounds to rapidly eliminate actions with low K‑step reward, thereby accelerating learning and improving finite‑horizon performance (Figs.~\ref{fig:results-synthetic}–\ref{fig:results-frozenlake-trading}).
In our tested environments, neither $\bm{\pi}^{\mathrm{1},\greedy}$ nor $\bm\pi^{\mathrm{2},\greedy}$ is optimal. In the long run, standard RL algorithms would eventually surpass LG1T/LG2T. To retain both fast initial convergence and asymptotic optimality, we introduce a hybrid framework (Algorithm~\ref{alg:lg1trl}) that transitions from LGKT to a standard RL method. As seen in Fig.~\ref{fig:results-riverswim}, this retains the early boost without sacrificing eventual performance.

\section*{Impact Statement}
This paper presents work whose goal is to advance the field
of Machine Learning. There are many potential societal
consequences of our work, none which we feel must be
specifically highlighted here.
\bibliography{ref}
\bibliographystyle{icml2026}
\newpage
\appendix
\onecolumn
\input{appendix}
\end{document}

%% file: intro.tex
% \vspace{-15pt}
\section{Introduction}
% \vspace{-5pt}
In many real-world sequential decision problems---from medical treatment regimens to financial trading sessions \citep{liu2020reinforcement,trella2025deployed,ghosh2024rebandit, hambly2023recent,liu2021finrl}---an agent must learn to perform well within a single, finite, and non-repeating trajectory. Formally, these are \emph{non-episodic, finite-horizon} Markov Decision Processes (MDPs), where the goal is to maximize cumulative reward up to a known terminal time without the benefit of environment resets.

Reinforcement learning (RL) theory, 
% \kyra{the position of these two citations is a bit strange} 
however, has largely developed around settings that differ fundamentally from this challenging regime. RL \citep{burnetas1997optimal,sutton1998reinforcement} has been studied under two predominant settings: (1) infinite-horizon RL, where algorithm performance is evaluated under either the \emph{average reward} \citep{boone2024achieving,agrawal2017optimistic,filippi2010optimism,fruit2020improved,talebi2018variance,fruit2018efficient,bartlett2012regal,zhang2023sharper,wei2020model,pmlr-v291-agrawal25a,pmlr-v97-lazic19a} or \emph{discounted cumulative reward} \citep{haarnoja2018soft,mnih2015human,schulman2015trust,schulman2017proximal} optimality criterion; and
(2) episodic finite-horizon RL, where the environment resets to a known initial state after a fixed number of steps
% , and the goal is to maximize the cumulative reward
\citep{jin2018q,zhang2020almost,NEURIPS2022_298c3e32,10.5555/3305381.3305409,efroni2019tight}.
These algorithmic tools are generally categorized as \emph{model-based}, which learn explicit transition dynamics to plan, or \emph{model-free}, which learn value functions or policies directly. 

% The misalignment in the algorithm development setting and usecase setting leads to poor finite-sample performance when applying existing RL algorithms to \emph{non-episodic finite-horizon} settings.
Consequently, applying standard RL algorithms to \emph{non-episodic finite-horizon} settings leads to poor finite-sample performance.
Model-based methods \citep{boone2024achieving,10.5555/3305381.3305409} are exponentially sample-inefficient in our setting as they must estimate the complete dynamics over the full horizon. 
% Meanwhile, 
Standard model-free, value-based methods \citep{zhang2020almost} face a fundamental barrier: they learn a Q-function that estimates returns until the terminal step. Without repeated episodes, accurately learning this full-horizon target from a single trajectory is inherently high-variance. Similarly, infinite horizon methods \citep{wei2020model} require the operating horizon to exceed the MDP's diameter or mixing time to guarantee convergence. Within a fixed horizon, convergence to the optimal policy can not be assured.
While deep learning extensions \citep{janner2019trust,kaiser2019model, mnih2015human,schulman2017proximal} enable scaling, their theoretical foundations—and thus their finite-sample behavior—rest on these tabular principles, making the tabular setting the necessary starting point for understanding and improving performance.

Our key insight is to address this estimation barrier by deliberately limiting the planning depth. Instead of targeting the full-horizon Q-function, we propose learning a K-step lookahead Q-function paired with an adaptive thresholding rule. This reduces the complexity of the value target while maintaining alignment with the long-term objective.
Our contributions are three-fold:
\begin{itemize}[leftmargin=12pt, itemsep=3pt, topsep=0pt, partopsep=0pt, parsep=0pt]
\item We introduce the K-step lookahead thresholding policy, a new policy class designed for sample-efficient learning in non-episodic finite-horizon MDPs (Sec.~\ref{subsec:k-step}).
% We note that w
When $K$ exceeds
% is larger than 
the total horizon $T$, 
% K-step lookahead thresholding
the policy is optimal, provided with a sufficiently high threshold. 
% We also show that 
Although when $K<T$, the policy may exhibit an optimality gap linear in $T$, we prove that it is optimal
% K-step lookahead thresholding policy has linear gap compared with the optimal policy, it is optimal 
for a two-state MDP under a stochastic dominance assumption (Assum.~\ref{assumption:stochastic-dominance}).
% \kyra{add guarantees on optimality for end-of-horizon?}.
% shift the target from estimating the $Q$-function to only estimating a K-step lookahead $Q$-function for finite horizon, non-episodic RL. Based on this, we propose a new K-step lookahead thresholding policy to improve the convergence rate in finite horizon.
\item We develop LGKT (\emph{LCB-Guided K-Step Thresholding}, Algorithm~\ref{alg:k-step}, Sec.~\ref{sec:online-algorithm}), a novel learning algorithm for non-episodic, finite horizon RL. We prove that when $K=1$, 
LGKT achieves minimax optimal constant regret (Theorem~\ref{th:one-step-regret}), 
% compared with one-step lookahead thresholding policy 
and for $K\geq 2$, its regret is $\mathcal{O}(\max((K-1),C_{K-1})\sqrt{SAT\log(T)})$ (Theorem~\ref{th:k-step-regret}) against the corresponding 
K-step benchmark where $C_{K-1}$ is an instance-dependent parameter (Sec.~\ref{sec:online-algorithm}).
% compared with any K-step lookahead thresholding policy. 
% Consequently, 
This yields a sublinear \emph{convergence rate} guarantee whenever $\max(K,C_{K-1})\sqrt{SA}=o(\sqrt{T})$, improving upon prior linear convergence results in our problem setting.
% \kyra{link to thm numbers in the above two points}
% guarantees a fast finite-sample convergence rate to K-step lookahead thresholding policy, 
% improving over the linear convergence rates of current methods to the optimal policy.
% \kyra{missing the point on showing fast finite-sample convergence}
\item 
% \kyra{missing metric: reward maximization} 
We evaluate LGKT's cumulative reward performance against state-of-the-art tabular RL methods (Sec.~\ref{sec:experiments}) on a suite of \emph{non-episodic environments}: 1) \textbf{1000 synthetic MDPs} with various state sizes, 2) \textbf{JumpRiverSwim} environments \citep{wei2020model}, and 3) \textbf{FrozenLake} \citep{brockman2016openaigym}.
% in an non-episodic setting. 
LGKT consistently achieves the highest cumulative reward. Due to the lack of available non-episodic RL environments, we extend our evaluation to one \emph{episodic environment} with continuous state spaces:
% dataset available, we also compare the performance on environments with continuous state space in episodic setting:
\textbf{AnyTrading} \citep{gym_anytrading}.
LGKT again achieves the highest cumulative reward, demonstrating its effectiveness across diverse RL regimes.
% Our method still achieves the highest cumulative reward, demonstrating the utility of our methods in a range of RL settings.
\end{itemize}

% \vspace{-5pt}
\subsection{Additional Related Work}
% \vspace{-5pt}
While the idea of using an approximate policy as a sample-efficient oracle is closely related to work on restless bandits (RB) \citep{xu2025restlesscontextualthresholdingbandit}, our setting is fundamentally different. RB considers binary actions, state-dependent rewards, and a per-round budget constraint across agents, which reduces the problem to comparing the relative value of actions among agents. In contrast, we study general finite-horizon MDPs with arbitrary action spaces, state-and-action-dependent rewards, and no budget constraint. Consequently, our algorithm must evaluate the absolute value of each action, leading to a distinct algorithmic approach.

\textbf{Thresholding Bandit/MDP\;\;} 
% \kyra{explain why it is related... k=1}
% We note that 
% When $K=1$, our problem is a thresholding contextual bandit problem where the context is generated according to an MDP. Thus, o
Our paper relates to regret-minimizing thresholding bandit problems, under which
% , where regret is defined by the gap between the arm's mean and a predefined threshold when the mean falls below the threshold.
% Under this setting, 
minimax optimal (max over all possible bandit instances and min over all possible learning algorithms) constant regret is achievable \citep{tamatsukuri2019guaranteed,michel2022regret,feng2025satisficingregretminimizationbandits}. 
Our problem setup differs from the traditional thresholding bandit framework by considering the problem in MDP and the reward structure.
% In contrast, 
\citet{hajiabolhassan2023online} study infinite-horizon thresholding MDPs, seeking policies with average rewards exceeding a predefined threshold. 
% We, however, focus on the finite-horizon setting and 
% % focus on
% designing an online policy that 
Our algorithm chooses actions whose K-step lookahead reward exceeds a threshold to improve finite-horizon performance.

% \textbf{Learning MDP With Lookahead Reward\;\;}
\textbf{Lookahead in RL\;\;}
Our work is distinct from two lines of research that also employ lookahead: (i) MDPs where the agent observes future rewards \citep{merlis2024reinforcement,pla2025hardness,merlis2024value}, and (ii) multi‑step lookahead value/policy iteration based on a K‑step Bellman operator in episodic \citep{efroni2020online} and infinite-horizon \citep{protopapas2024policy,efroni2018multiple,bonet2000planning} settings. 
In contrast, we use lookahead to denote a truncated Q‑function for sample‑efficient online learning in non‑episodic, finite‑horizon settings.
% A line of research studies MDPs where the agent has lookahead access to future rewards, meaning that at each time step, the rewards for the next 
% K steps are observable \citep{merlis2024reinforcement,pla2025hardness,merlis2024value}. In contrast, we use
% % our work uses 
% the term lookahead to refer to
% % differently: we refer specifically to 
% a truncated Q-function that estimates returns over a limited future horizon.
% % , rather than to the observation of predetermined future rewards.

% \textbf{Multi-Step Lookahead Value/Policy Iteration\;\;}
% Our work relates to multi-step lookahead value and policy iteration, which utilizes a K-step Bellman operator for value function estimation rather than the standard one-step operator \citep{efroni2020online,protopapas2024policy,efroni2018multiple,bonet2000planning}. However, our approach diverges in two key aspects: First, we use lookahead to denote the use of a truncated Q-function to design online algorithms. Second, we study a non-episodic, finite-horizon setting, which differs from the episodic setting in \citet{efroni2020online} and the infinite-horizon settings examined in \citet{protopapas2024policy,efroni2018multiple}.

\textbf{Approximate Oracle MDP Solutions\;\;}
With known MDP models,
% rewards and transitions,
approximate solutions are used to
% have been proposed to
mitigate the computational burden of 
planning in
% solving MDPs with
% solve MDPs to deal with the computational cost caused by a 
large state and action spaces. \citet{bertsekas2022abstract} proposes K-step lookahead schemes that select actions by maximizing $\E\Large[\sum_{t=0}^{K-1}R(s_t,a_t)+\tilde V(s_k)\mid a_0=a\Large]$, where $\tilde V$ is a user-specified lookahead function, and bounds
% that needs to be specified.
% characterizes 
the suboptimality gap of such policies in terms of the quality of $\tilde V$. \citet{de2003linear,bertsekas1996temporal} used linear function approximation 
of  
% to estimate
the value functions and \citet{munos2003error} proposed approximate policy iteration. Theoretical guarantees 
% of the above methods
are provided in a discounted reward setting, relying on discounted contraction. 
While these solutions have inspired more efficient online learning algorithms \citep{pmlr-v97-lazic19a},
these algorithms still target the optimal policy, creating
% This creates 
a fundamental learning bottleneck (Sec.~\ref{subsec:hardness}).
In contrast, we propose to learn a truncated Q-function, easing the learning objective.
Our K-step lookahead thresholding policy shares a similar idea with K-step lookahead schemes \citep{bertsekas2022abstract} with $\tilde V=0$. 

\paragraph{Conflict of Interest Disclosure} We do not have any conflict of interest to disclose.
% the principles from these approximate dynamic programming methods provide key insights for learning an optimal policy in an online setting, a fundamental limitation remains: the regret lower bound is $\Omega(T)$ in non-episodic finite horizon setting \citep{NIPS2008_e4a6222c} if the oracle policy is chosen to be the optimal policy. 
% Our work are motivated by carefully choosing approximate policy that can serve as an attainable benchmark, enabling the design of online algorithms with fast convergence rates. Therefore, we introduce

% We also provide an online algorithm that has fast convergence.
% Our work learns a K-step lookahead thresholding policy that conceptually aligns with traditional K-step lookahead schemes with the lookahead function being 0. \citet{bertsekas2022abstract} provides the error bound for multi-step lookahead policy in an discounted setting and the proof relies on discounted contraction. We focus on an undiscounted setting instead. Other approximate solutions of MDP include linear function approximation of value functions \citep{de2003linear} and approximate policy iteration\citep{munos2003error}. These methods, however, are not typically designed to serve as the oracle within an online algorithm. Instead, the online algorithms used similar ideas to design algorithms with unknown reward and transitions \citep{wei2020model,protopapas2024policy}.

%% file: appendix.tex
\renewcommand{\thealgorithm}{\thesection.\arabic{algorithm}}
\section{Additional Alorithms}
\begin{algorithm}[H]
\caption{Estimate-$r^\step$}
\label{alg:subroutine-estimate-k-1}
    \begin{algorithmic}[1]
        \STATE{\bfseries Input:} sampling algorithm: $\mathrm{ALG}_{\step}$, reference state and action: $s_0,a_0$,initial state: $s'$ current step: $t$, cumulative reward 1- and K-step reward: $\hat{\phi}^\mathrm{1}(s,a),\hat{\phi}^\step(s,a)$, $N_{s,a}^{(t)},N_{s,a,\step}^{(t)}$
        \STATE{Initialize:} $s_t=s'$
        \FOR{$k=t,t+1,t+K-2$}
       \STATE Retrieve $s_k$, set $a_k=\mathrm{ALG_K}(s_k|s,a)$\;
       \STATE Observe $(r_k,s_{k+1})$\;
       \STATE Update the cumulative one-step reward as follows:
       \STATE \hspace{1em} $\hat{\phi}^{\mathrm{one}}(s_k, a_k) = \hat{\phi}^{\mathrm{1}}(s_k, a_k) +
        % r^m_t\mathbb{I}\left\{a_{t}^m = 1\right\}$\;
        r_{k}$\;
        \STATE Update the lower confidence bound of the K step reward as follows:
        \STATE \hspace{1em}  $N_{s, a}^{(k+1)} = 
        N_{s, a}^{(k)} + \mathbb{I}
        \left\{a_{k} = a,s_k=s\right\}$, $N_{s,a,\step}^{(k+1)}=N_{s,a,\step}^{(k)}$\;
        \STATE \hspace{1em} $\hat{r}^{\mathrm{1}}_{s,a}(k+1)=\frac{\hat{\phi}^\mathrm{1}(s, a)}{N_{s, a}^{(k+1)}}$, $\hat{r}_{s,a}^\step(k+1)=\hat{r}_{s,a}^\step(k)$\;
        \STATE  \hspace{1em} $\mathrm{LCB}^{(k+1)}_{s,a} = \hat{r}^{\mathrm{1}}_{s,a}(k+1)+\hat{r}_{s,a}^\step(k+1) - \sqrt{\frac{g\left(N_{s, a}^{(k+1)} + 2\right)}{N_{s, a}^{(k+1)} + 2}}-\sqrt{\frac{g\left(N_{s, a,\step}^{(k+1)} + 2\right)}{N_{s, a,\step}^{(k+1)} + 2}}$\;
       \ENDFOR
       \STATE Update the cumulative $\mathrm{K}-1$ step reward of $s_0,a_0$ as follows:
       \STATE \hspace{1em} $\hat{\phi}^\step(s_0,a_0)=\hat{\phi}^\step(s_0,a_0)+\sum_{k=t}^{t+K-2} r_k$\;
       \STATE Update the lower confidence bound of the K step reward as follows:
       \STATE \hspace{1em} $N_{s_0,a_0,\step}^{(t+K-1)}=N_{s_0,a_0,\step}^{(t+K-1)}+1$\;
       \STATE \hspace{1em} $\hat{r}_{s_0,a_0}^\step(t+K-1)=\frac{\hat{\phi}^\step(s_0,a_0)}{N_{s_0,a_0,\step}^{(t+K-1)}}$\;
       \STATE  \hspace{1em} $\mathrm{LCB}^{(t+K-1)}_{s_0,a_0} = \hat{r}^{\mathrm{1}}_{s_0,a_0}(t+K-1)+\hat{r}_{s_0,a_0}^\step(t+K-1) - \sqrt{\frac{g\left(N_{s_0, a_0}^{(t+K-1)} + 2\right)}{N_{s_0, a_0}^{(t+K-1)} + 2}}-\sqrt{\frac{g\left(N_{s_0, a_0,\step}^{(t+K-1)} + 2\right)}{N_{s_0, a_0,\step}^{(t+K-1)} + 2}}$\;
       \STATE{\bfseries Return:} $N_{s,a,\step}^{(t+K-1)},N_{s,a}^{(t+K-1)},\hat{\phi}^\step,\hat{\phi}^\mathrm{1},\mathrm{LCB}^{(t+K-1)}_{s,a}$
    \end{algorithmic}
\end{algorithm}
\begin{algorithm}[H]
\caption{Update-LCB-1}
\label{alg:update_lcb_1}
    \begin{algorithmic}
    \STATE {\bfseries{Input:} }cumulative reward: $\hat{\phi}(s,a)$, $N_{s,a}$, reward: $r$, state: $s_t$, action: $a_t$\;
         \STATE Update the lower confidence bound of the reward as:\;
            % follows:
        
        \STATE \hspace{1em}$\hat{\phi}(s_t,a_t)=\hat{\phi}(s_t,a_t)+r$\;
                    \STATE \hspace{1em} $N_{s, a} = 
        N_{s, a} + \mathbb{I}
        \left\{a_{t} = a,s_t=s\right\}$\label{eq:update_count}\;
       \STATE\hspace{1em} $\hat{r}_{s,a}^\mathrm{1} ={\hat{\phi}(s, a)}/{N_{s, a}} $\label{alg:update-sample-mean}\;
        \STATE \hspace{1em}$\mathrm{LCB}_{s,a} = \hat{r}_{s,a}^\mathrm{1} - { \sqrt{\frac{g\big(N_{s, a} + 2\big)}{N_{s, a} + 2}}}$\;\label{eq:update_LCB}
            
                \STATE{\bfseries Return:} $\hat{\phi}(s,a),N_{s,a},\mathrm{LCB}_{s,a},\hat{r}_{s,a}^\mathrm{1}$
    \end{algorithmic}
\end{algorithm}
\begin{algorithm}[H]
\caption{Update-LCB-K}
\label{alg:update_lcb_k}
    \begin{algorithmic}
    \STATE{\bfseries Input:} cumulative 1-step lookahead reward: $\hat{\phi}^1_{s,a}$, cumulative K$-1$-step lookahead reward: $\hat{\phi}^\step_{s,a}$, $N_{s,a},N_{s,a,\step}$, reward: $r$, state: $s_t$, action: $a_t$\;

        \STATE Update the lower confidence bound of the K step reward as follows:
        \STATE \hspace{1em} $\hat{\phi}^{\mathrm{1}}_{s_t, a_t} = \hat{\phi}^{\mathrm{1}}_{s_t, a_t} +
        % r^m_t\mathbb{I}\left\{a_{t}^m = 1\right\}$\;
        r$\;
        \STATE \hspace{1em}  $N_{s, a} = 
        N_{s, a} + \mathbb{I}
        \left\{a_{t} = a,s_t=s\right\}$\;
        \STATE \hspace{1em} $\hat{r}^{\mathrm{1}}_{s,a}=\frac{\hat{\phi}^\mathrm{1}_{s, a}}{N_{s, a}}$, $\hat{r}_{s,a}^\step=\frac{\hat\phi^\step_{s,a}}{N_{s,a,\step}}$\;
        \STATE  \hspace{1em} $\mathrm{LCB}_{s,a} = \hat{r}^{\mathrm{1}}_{s,a}+\hat{r}_{s,a}^\step - \sqrt{\frac{g\left(N_{s, a} + 2\right)}{N_{s, a} + 2}}-\sqrt{\frac{g\left(N_{s, a,\step} + 2\right)}{N_{s, a,\step} + 2}}$\;
        \STATE {\bfseries Return:} $\hat{\phi}^1(s,a),N_{s,a},N_{s,a,\step}, \mathrm{LCB}_{s,a}, \hat{r}^\mathrm{1}_{s,a},\hat{r}^\step_{s,a}$
    \end{algorithmic}
\end{algorithm}
\begin{algorithm}[H]
\caption{LCB-Guided 1-2 Step Thresholding (LG1-2T)}
\label{alg:lg1-2t}
    \begin{algorithmic}[1]
\STATE {\bfseries Input:} initial state $s$, cutoff: $t_c$\;
\FOR{$t=0,1,\cdots,T$}
\IF{$t\leq t_c$}
\STATE Run LG1T (Algorithm~\ref{alg:'name'})\;
\ELSE
\STATE Run LG2T (Algorithm~\ref{alg:k-step})\;
\ENDIF
\ENDFOR
    \end{algorithmic}
\end{algorithm}
\begin{algorithm}[H]
\caption{LCB-Guided 1 Step Thresholding-RL (LG1T-RL)}
\label{alg:lg1trl}
    \begin{algorithmic}[1]
\STATE {\bfseries Input:} initial state $s$, cutoff: $t_c$\;
\FOR{$t=0,1,\cdots,T$}
\IF{$t\leq t_c$}
\STATE Run LG1T (Algorithm~\ref{alg:'name'})\;
\ELSE
\STATE Run PMEVI-KLUCRL \citep{boone2024achieving}\;
\ENDIF
\ENDFOR
    \end{algorithmic}
\end{algorithm}
\section{Technical Details of Section~\ref{subsec:optimality}}\label{app:optimality}
\subsection{Proof of Theorem~\ref{th:optimality-two-state}}\label{proof:optimality-two-state}
We will first use induction to prove that $\bm\pi^{\mathrm{1},\greedy}=\bm\pi^*$. To ease notation, we will use $V^\greedy$ to denote $V^{\bm\pi^{\mathrm{1},\greedy}}$. 

We will consider the following two cases: 1) $P_{a_{1,1}^*}(1,1)\geq P_{a_{0,1}^*(0,1)}$ and 2) $P_{a_{1,1}^*}(1,1)\leq P_{a_{0,1}^*(0,1)}$.
\paragraph{Case 1:}
The induction hypothesis is:
\begin{align}
    \forall k\leq T, &V_k^{\greedy}(1)\geq V_k^\greedy(0)\label{eq:monotone-v}\\
    &\pi_{k}^{\mathrm{1},\greedy}=\pi_k^*. \label{eq:optimal-one}
\end{align}
When $k=T$ this holds trivially, suppose it holds for $k\geq t+1$, when $k=t$: By optimal Bellman equation, we only need to prove the following to show Eq.~\eqref{eq:optimal-one}:
\begin{equation*}
    r(s,a_{s,1}^*)+\sum_{s'}P_{a_{s,1}^*}(s,s')V^\greedy_{t+1}(s')\geq r(s,a)+\sum_{s'}P_{a}(s,s')V^\greedy_{t+1}(s').
\end{equation*}
We have
\begin{align*}
   & r(s,a_{s,1}^*)+\sum_{s'}P_{a_{s,1}^*}(s,s')V^\greedy_{t+1}(s')-r(s,a)-\sum_{s'}P_{a}(s,s')V^\greedy_{t+1}(s')\\
   &=r(s,a_{s,1}^*)-r(s,a)+\left(P_{a_{s,1}^*}(s,s')-P_{a}(s,s')\right)\left(V_{t+1}^\greedy(1)-V_{t+1}^\greedy(0)\right)\\
   &\geq 0,
\end{align*}
where the last inequality holds because of Assumption~\ref{assumption:stochastic-dominance} and induction hypothesis. This shows Eq.~\eqref{eq:optimal-one}. For Eq.~\eqref{eq:monotone-v}, we have
\begin{align*}
    V_k^\greedy(1)-V_k^\greedy(0)&=r(1,a_{1,1}^*)-r(0,a_{0,1}^*)+\left(P_{a_{1,1}^*}(1,1)-P_{a_{0,1}^*}(0,1)\right)\left(V_{k+1}^\greedy(1)-V_{k+1}^\greedy(0)\right)\\
    &\geq 0.
\end{align*}
\paragraph{Case 2:}
The induction hypothesis is:
\begin{align}
    \forall k\leq T, &0\leq V_k^{\greedy}(1)- V_k^\greedy(0)\leq r(1,a_{1,1}^*)-r(0,a_{0,1}^*)\label{eq:monotone-v-2}\\
    &\pi_{k}^{\mathrm{1},\greedy}=\pi_k^*. \label{eq:optimal-one-2}
\end{align}
This holds trivially for $k=T$. Suppose it holds for $k\geq t+1$, when $k=t$: Similar to Case 1, we can show Eq.~\eqref{eq:optimal-one-2} holds. We will first show the RHS of Eq.~\eqref{eq:monotone-v-2}:
\begin{align*}
    V_k^\greedy(1)-V_k^\greedy(0)&=r(1,a_{1,1}^*)-r(0,a_{0,1}^*)+\left(P_{a_{1,1}^*}(1,1)-P_{a_{0,1}^*}(0,1)\right)\left(V_{k+1}^\greedy(1)-V_{k+1}^\greedy(0)\right)\\
    &\leq r(1,a_{1,1}^*)-r(0,a_{0,1}^*),
\end{align*}
where the inequality holds because of induction hypothesis and $P_{a_{1,1}^*}(1,1)-P_{a_{0,1}^*}(0,1)\leq 0$. For LHS, we have
\begin{align*}
    V_k^\greedy(1)-V_k^\greedy(0)&=r(1,a_{1,1}^*)-r(0,a_{0,1}^*)+\left(P_{a_{1,1}^*}(1,1)-P_{a_{0,1}^*}(0,1)\right)\left(V_{k+1}^\greedy(1)-V_{k+1}^\greedy(0)\right)\\
    &\geq r(1,a_{1,1}^*)-r(0,a_{0,1}^*)+\left(P_{a_{1,1}^*}(1,1)-P_{a_{0,1}^*}(0,1)\right)(r(1,a_{1,1}^*)-r(0,a_{0,1}^*))\\
    &=\left(P_{a_{1,1}^*}(1,1)+P_{a_{0,1}^*}(0,0)\right)(r(1,a_{1,1}^*)-r(0,a_{0,1}^*))\\
    &\geq 0.
\end{align*}
This proves $\bm\pi^{\mathrm{1},\greedy}=\bm\pi^*$. Therefore, we have $\arg\max Q_k^*(s,a)=a_{s,1}^*$ for any $k\geq 0$ since $\bm\pi^{\mathrm{1},\greedy}=\bm\pi^*$. By definition of $\bm\pi^{\steps,\greedy}$ (Eq.~\eqref{eq:K-step-lookahead-greedy-policy}), we have $\bm\pi^{\steps,\greedy}(a_{s,1}^*|s)=1=\bm\pi^{\mathrm{1},\greedy}(a_{s,1}^*|s)$. This shows that $\bm\pi^{\steps,\greedy}=\bm\pi^{\mathrm{1},\greedy}=\bm\pi^*$. This concludes the proof.
\subsection{Proof of Theorem~\ref{th:lower-bound-linear}}
\label{proof:lower-bound-linear}
We will let the state space be $\{B,G,D_1,D_2,\cdots,D_{S-2}\}$. And let the action space be $\{a_0,a_1,a_2,\cdots,a_{A-3}\}$. We will let $a_2,\cdots,a_{A-3}$ to have the same behaviour as $a_1$ and $D_1,D_2,\cdots,D_{S-2}$ have the same behaviour as $G$. We will define the transition matrix as:
\begin{equation*}
    P_{a_0}(B,B)=1,P_{a_1}(B,G)=P_{a_1}(B,D_i)=1/S-1, P_{a_0}(G,G)=P_{a_0}(G,D_i)=1/S-1,P_{a_1}(G,B)=1.
\end{equation*}
The reward $R$ is defined as:
\begin{equation*}
    R_{B,a_0}=-1,R_{B,a_1}=-(K+1),R_{G,a_0}=0,R_{G,a_1}=-1.
\end{equation*}
Therefore, it is easy to see that for $h\leq T-k+1$, $Q_h^*(B,a_0)=T-h+1, Q_h^*(B,a_1)=-(K+1)$. Therefore, K-step lookahead thresholding policy will always choose $a_0$ at state $B$ when $\bm\gamma$ is high. However, this will lead to $V_0^{\bm\pi^{\steps,\greedy}}(B)=-T-1$ while $V_0^*(B)=-K-1$. Therefore, the gap is $T-K$, which will grow linearly in $T$.

When threshold is low, i.e. $\gamma\leq -(K+1)$, $\bm\pi^{\mathrm{K},\bm\gamma}$ will become random policy, and it is easy to see that the gap also grows linearly in $T$.
\subsection{Proof of Theorem~\ref{th:instance-lower-bound}}\label{sec:proof-instance}
We will use induction on the number of time until the end of horizon. Specifically, the hypothesis is
\begin{equation*}
    V_{h}^*(s)-V_h^{\bm\pi^{\steps,\bm\gamma}}(s)\leq \sum_{t=h}^{T-K}\max_{s'}\left|P_{\pi_t^{\steps,\bm\gamma}}\left(s'\right)-P_{a_{s',T-t+1}^*\left(s'\right)}\right|_1\max_{s'}V_{t+1}^*\left(s'\right).
\end{equation*}
When $h\geq T-K+1$, the hypothesis holds because the K-step lookahead policy is exactly the optimal policy. Suppose this holds when $h\geq l+1$, when $h=l$,
\begin{align*}
    V_{l}^*(s)-V_l^{\bm\pi^{\steps,\bm\gamma}}(s)&=R(s,a_{s,T-l+1}^*-\sum_{a}\pi_l^{\steps,\bm\gamma}(a|s)R(s,a)\\
    &+\sum_{s'}P_{a_{s,T-l+1}^*}(s,s')V_{l+1}^*(s')-\sum_{s'}\sum_{a}\pi_l^{\steps,\bm\gamma}(a|s)P_{a}(s,s')V_{l+1}^{\bm\pi^{\steps,\bm\gamma}}(s')\\
    &\leq \sum_{s'}\left(P_{a_{s,T-l+1}^*}(s,s')-\sum_{a}\pi_l^{\steps,\bm\gamma}(a|s)P_{a}(s,s')\right)V_{l+1}^*(s')\\
    &+\sum_{s'}\sum_{a}\pi_l^{\steps,\bm\gamma}(a|s)P_{a}(s,s')\left(V_{l+1}^*(s')-V_{l+1}^{\bm\pi^{\steps,\bm\gamma}}(s')\right)\\
    &\leq \left|\sum_{s'}\left(P_{a_{s,T-l+1}^*}(s,s')-\sum_{a}\pi_l^{\steps,\bm\gamma}(a|s)P_{a}(s,s')\right)\right|\max_{s'}\left|V_{l+1}^*(s')\right|\\
    &+\max_{s'} V_{l+1}^*(s')-V_{l+1}^{\bm\pi^{\steps,\bm\gamma}}(s')\\
    &\leq \sum_{t=l}^{T-K}\max_{s'}\left|P_{\pi_t^{\steps,\bm\gamma}}\left(s'\right)-P_{a_{s',T-t+1}^*\left(s'\right)}\right|_1\max_{s'}V_{t+1}^*\left(s'\right),
\end{align*}
where the last inequality holds because of the induction hypothesis. This completes the proof.
\section{Discussions Of Section~\ref{sec:online-algorithm}}
\subsection{Discussion On Assumption~\ref{assumption:low-regret-algorithm}}\label{appendix:discussion-low-regret-algorithm}
When $K=2$, then Assumption~\ref{assumption:low-regret-algorithm} is equivalent to assuming the regret of the contextual bandit problem is sublinear. Therefore, UCB algorithm satisfies the assumption with $C_{K-1}=A$ \citep{lattimore2020bandit}. When $K>2$, Assumption~\ref{assumption:low-regret-algorithm} relates with regret minimization in episodic RL with horizon length $K-1$ and total number of episodes $H$. Specifically, episodic RL with horizon length $K-1$ is maximizing $\E_{s_0\sim P}\left[\sum_{k=0}^{K-2} R_{s_k,a_k}\mid s_0\right]$ within $H$ episodes. Further, the regret for episodic RL is defined as $\max \sum_{h=0}^{H-1}\E_{s_0\sim P}\left[\sum_{k=0}^{K-2} R_{s_k,a_k}\right]-\sum_{k=0}^{K-2} R_{s_k^h,a_k^h}$ which mathces our assumption,. Therefore, algorithms such as UCBVI-BF \citep{10.5555/3305381.3305409} satisfies the assumption with $C_{K-1}=K^2SA$. 

We emphasized that though episodic RL algorithms with horizon length $K-1$ can be used as the subroutine for Algorithm~\ref{alg:k-step}, it is not exclusive to episodic RL algorithms as any algorithms satisfying Assumption~\ref{assumption:low-regret-algorithm} is applicable. Further, Algorithm~\ref{alg:k-step} is fundamentally different from episodic RL algorithm even when episodic RL algorithm is used as subroutine. As shown in Algorithm~\ref{alg:k-step}, the subroutine is merely used to calculate $\mathrm{K-1}$-step reward and will only be triggered in $\mathcal{O}(K\sqrt{T})$ times. Moreover, the goal of Algorithm~\ref{alg:k-step} is choosing actions whose K-step lookahead reward is above the threshold while episodic RL algorithms will choose the action that maximizes the end-of-horizon lookahead reward. The fundamental difference in the learning goal makes our algorithm distinct from episodic RL algorithms.

\section{Technical Details of Section~\ref{sec:online-algorithm}}
Before proving Theorem , we first introduce Lemma \ref{le:regret-decomposition} which decomposes the cumulative regret to per-timestep regret.
\begin{lemma}\label{le:regret-decomposition}
    For any $t\leq T$, let per-timestep regret $\text{Regret}_t(\pi,\mathcal{H}_t)$ for the policy $\bm{\pi}$ be defined as:
    \begin{align*}
\text{Regret}_t(\bm{\pi},\mathcal{H}_t)&=\E_{\pi_{t}}\left[c_{t}(a_{t})\mid\mathcal{H}_t\right]+\E_{\pi_{t}}\left[\E_{\bm{\pi}^{\mathrm{1},\bm\gamma}}\left[\sum_{k=t+1}^T c_k\left(a_{k}\right)|{r}_{t},{s}_{t+1},{a}_{t},\mathcal{H}_t\right]\right]\\&-\E_{\pi_t^{\mathrm{1},\bm\gamma}}\left[c_{t}\left(a_{t}\right)\mid \mathcal{H}_t\right]-\E_{\pi_t^{\mathrm{1},\bm\gamma}}\left[\E_{\bm{\pi}^{\mathrm{1},\bm\gamma}}\left[\sum_{k=t+1}^T c_k\left(a_{k}\right)|{r}_{t},{s}_{t+1},{a}_{t},\mathcal{H}_t\right]\right].
    \end{align*} 
    Then the cumulative regret $\mathcal{R}(\pi)$ can be decomposed as following:
    \begin{equation*}
        \mathcal{R}(\bm\pi)=\E\left[\sum_{t=0}^T\text{Regret}_t\left(\bm\pi,\mathcal{H}_t\right)\right].
    \end{equation*}
\end{lemma}
\subsection{Proof of Lemma \ref{le:regret-decomposition}}
\begin{proof}
We will prove the result iteratively from $t=1$ to $T$. All we need is calculation. 
    We have
    \begin{align*}
        \mathcal{R}(\bm{\pi})&=\E_{\bm{\pi}}\left[\sum_{t=0}^T c_t(a_{t})|{s}_0\right]-\E_{\bm{\pi}^{\mathrm{1},\bm\gamma}}\left[\sum_{t=0}^T c_t(a_{t})|{s}_0\right]\\
&=\E_{\pi_0}\left[c_0(a_{0})\mid\mathcal{H}_0\right]+\E_{\pi_0}\left[\E_{\bm{\pi}}\left[\sum_{k=1}^T c_k\left(a_{k}\right)|{r}_0,{s}_1,{a}_1,\mathcal{H}_0\right]\right]\\
        &-\E_{\pi_0}\left[c_{0}(a_0)\mid\mathcal{H}_0\right]-\E_{\pi_0}\left[\E_{\bm{\pi}^{\mathrm{1},\bm\gamma}}\left[\sum_{k=1}^T c_k\left(a_{k}\right)|{r}_0,{s}_1,{a}_0,\mathcal{H}_0\right]\right]\\
        &+\E_{\pi_0}\left[c_{0}(a_0)\mid\mathcal{H}_0\right]+\E_{\pi_0}\left[\E_{\bm{\pi}^{\mathrm{1},\bm\gamma}}\left[\sum_{k=1}^T c_k\left(a_{k}\right)|{r}_0, {s}_1,{a}_0,\mathcal{H}_0\right]\right]\\
        &-\E_{\pi_0^{\mathrm{1},\bm\gamma}}\left[ c_{0}\left(a_0\right)\mid\mathcal{H}_0\right]-\E_{\pi_0^{\mathrm{1},\bm\gamma}}\left[\E_{\bm{\pi}^{\mathrm{1},\bm\gamma}}\left[\sum_{k=1}^T c_k\left(a_{k}\right)|{r}_0,{s}_1,{a}_0,\mathcal{H}_0\right]\right]\\
        &=\E\left[\E_{\bm{\pi}}\left[\sum_{k=1}^T c_k\left(a_{k}\right)|\mathcal{H}_1\right]-\E_{\bm{\pi}^{\mathrm{1},\bm\gamma}}\left[\sum_{k=1}^T c_k\left(a_{k}\right)|\mathcal{H}_1\right]\right]\\
        &+\text{Regret}\left(\bm\pi,\mathcal{H}_0\right).
    \end{align*}
    We can continue this trick to $\E_{\bm{\pi}}\left[\sum_{k=1}^T c_k\left(a_{k}\right)|\mathcal{H}_1\right]-\E_{\bm{\pi}^{\mathrm{1},\bm\gamma}}\left[\sum_{k=1}^T c_k\left(a_{k}\right)|\mathcal{H}_1\right]$, this will give us the desired result.
\end{proof}
\subsection{Proof of Theorem~\ref{th:one-step-regret}}\label{appendix:proof-of-one-step-theorem}
Before we prove the theorem, we first provide a technical lemma showing the concentration property of the empirical mean estimator $\hat{r}_{s,a}^{(t)}$, which is a direct result from \citealt{howard2021time}.
\begin{lemma}\label{le:concentration-empirical}
    For any $\alpha>0$, $0\leq t\leq T$, we have
    \begin{equation*}
        \prob\left(\left|\hat{r}_{s,a}^{(t)}-r_{s,a}^{(t)}\right|\geq 1.7\sqrt{\frac{\log\log(2N_{s,m}(t))+0.72\log(10.4/\alpha)}{N_{s,a}(t)}}\right)\leq \alpha.
    \end{equation*}
\end{lemma}
\begin{proof}
We will begin by the proof sketch: The proof proceeds in three steps 1) We show that analyzing the regret is equivalent to analyzing the time that we pull bad actions. 2) We show that pulling bad arms is caused by either overestimation of the reward for bad arms or underestimation of the reward for good arms. 3) We apply concentration inequality which does not require stationarity of the reward distribution (Lemma~\ref{le:concentration-empirical}). Moreover, with our design of lower confidence bound, we show that the time of overestimation is at most a constant. 
   
   The proof will proceed by first showing that the cumulative regret can be decomposed into per-timestep regret $\text{Regret}_t(\bm{\pi},\mathcal{H}_t)$. 
   By Lemma~\ref{le:regret-decomposition}, we only need to bound $\text{Regret}_t(\pi,\mathcal{H}_t)$. We have
    \begin{align*}
         % \text{Regret}_t(\pi,\mathcal{H}_t)&=\E\left[\sum_{m=1}^M c_{t}^m(\pi_{t+1})\right]+\E_{\substack{{s}'\sim q(\cdot|{s}_t,\pi_{t+1}({s}_t))\mathstrut\\\pi({s}_t)\sim \prob}}\left[\E_{\bm{\pi}^*}\left[\sum_{m=1}^M\sum_{k=t+1}^T c_k^m\left(\bm{\pi}^*\right)|{s}_{t+1}={s}',{a}_{t+1}=\pi_{t+1}({s}_t),\mathcal{H}_{t}\right]\right]\\&-\E\left[\sum_{m=1}^M c_{t}^m\left(\bm{\pi}^*\right)\right]-\E_{\substack{{s}'\sim q(\cdot|{s}_t,\bm{\pi}^*({s}_t))\\\bm{\pi}^*({s}_t)\sim \prob}}\left[\E_{\bm{\pi}^*}\left[\sum_{m=1}^M\sum_{k=t+1}^T c_k^m\left(\bm{\pi}^*\right)|{s}_{t+1}={s}',{a}_{t+1}=\bm{\pi}^*({s}_t),\mathcal{H}_{t}\right]\right]\\
         % &=\E\left[\sum_{m=1}^M c_{t}^m(\pi_{t+1})\right]+\E_{\substack{{s}'\sim q(\cdot|{s}_t,\pi_{t+1}({s}_t))\mathstrut\\\pi({s}_t)\sim \prob}}\left[\E_{\bm{\pi}^*}\left[\sum_{m=1}^M\sum_{k=t+1}^T c_k^m\left(\bm{\pi}^*\right)|{s}_{t+1}={s}',{a}_{t+1}=\pi_{t+1}({s}_t),\mathcal{H}_{t}\right]\right]\\
         \text{Regret}_t(\pi,\mathcal{H}_t)&=\E_{\pi_{t}}\left[ c_{t}(a_{t})\mid\mathcal{H}_t\right]+\E_{\pi_{t}}\left[\E_{\bm{\pi}^{\mathrm{1},\bm\gamma}}\left[\sum_{k=t+1}^T c_k\left(a_{k}\right)|{r}_{t},{a}_{t},s_{t+1},\mathcal{H}_t\right]\right]\\&-\E_{\pi_t^{\mathrm{1},\bm\gamma}}\left[ c_{t}\left(a_{t}\right)\mid \mathcal{H}_t\right]-\E_{\pi_t^{\mathrm{1},\bm\gamma}}\left[\E_{\bm{\pi}^{\mathrm{1},\bm\gamma}}\left[\sum_{k=t+1}^T c_k\left(a_{k}\right)|{r}_{t},{s}_{t+1},{a}_{t},\mathcal{H}_t\right]\right]\\
         &=\E_{\pi_{t}}\left[c_{t}(a_{t})\mid\mathcal{H}_t\right]+\E_{\pi_{t}}\left[\E_{\bm{\pi}^{\mathrm{1},\bm\gamma}}\left[\sum_{k=t+1}^T c_k\left(a_{k}\right)|{r}_{t},{s}_{t+1},{a}_{t},\mathcal{H}_t\right]\right]\\
         &\leq \E\left[\sum_{\Delta_{s,a}^\mathrm{1}>0}\Delta_{s,a}^\mathrm{1}\mathbb{I}\left\{a_{t}=a,s_t=s\right\}\right]+\E_{\pi_{t+1}}\left[V_{t+1}^{\bm{\pi}^{\mathrm{1},\bm\gamma}}\left({s}_{t+1}\right)\mid {s}_{t+1},{a}_{t},\mathcal{H}_t\right]\\
         &=\E\left[\sum_{\Delta_{s,a}^\mathrm{1}>0}\Delta_{s,a}^\mathrm{1}\mathbb{I}\left\{a_{t}=a,s_t=s\right\}\right],
    \end{align*}
    where the second and the third equality holds by Assumption~\ref{assumption:exists-a-good-action}. Then we have the cumulative regret is upper bounded by
    \begin{equation*}
        \sum_{\Delta_{s,a}^\mathrm{1}>0}\Delta_{s,a}^\steps\E\left[\sum_{t=0}^T \mathbb{I}\left\{a_{t}=a,s_t=s\right\}\right].
    \end{equation*}
       We have
\begin{align}
    \E\left[\sum_{t=0}^T \mathbb{I}\left\{a_{t}=a,s_t=s\right\}\right]&=\E\left[\sum_{t=0}^T \mathbb{I}\left\{a_{t}=a,s_t=s,\mathrm{LCB}_{s,a}^{(t)}\geq\gamma\right\}\right]\label{eq:first-term-regret-1}\\&+\E\left[\sum_{t=0}^T \mathbb{I}\left\{a_{t}=a,s_t=s,\mathrm{LCB}_{s,a}^{(t)}\leq\gamma\right\}\right]\label{eq:second-term-regret-1}.
\end{align}
We will bound the two terms separately. For Eq.~\eqref{eq:first-term-regret-1}, we have
\begin{align}
    &\E\left[\sum_{t=0}^T \mathbb{I}\left\{a_{t}=a,s_t=s,\mathrm{LCB}_{s,a}^{(t)}\geq\gamma\right\}\right]\notag\\&=\E\left[\sum_{k=1}^\infty \mathbb{I}\left\{\mathrm{LCB}_{s,a}^{(\tau_{s,a}^k-1)}\geq\gamma,\tau_{s,a}^k\leq T\right\}\right]\notag\\
    &\leq \E\left[\sum_{k=1}^T \mathbb{I}\left\{\hat{r}_{s,a}^\mathrm{1}{(\tau_{s,a}^k-1)}\geq\gamma+\sqrt{\frac{3\log\left(N_{s,a}^{\left(\tau_{s,a}^k\right)}\right)}{N_{s,a}^{\left(\tau_{s,a}^k\right)}}}\right\}\right]\notag\\
    &\leq\sum_{k=1}^T \exp\left(-N_{s,a}^{\left(\tau_{s,a}^k\right)}\left(\Delta_{s,a}^\mathrm{1}\right)^2+\log\log\left(N_{s,a}^{\left(\tau_{s,a}^k\right)}\right)-3\log\left(N_{s,a}^{\left(\tau_{s,a}^k\right)}\right)\right)\notag\\
    &\leq \sum_{k=1}^T \frac{1}{k^2}\exp\left(-k\left(\Delta_{s,a}^\mathrm{1}\right)^2\right)\notag\\
    &\leq 2,\label{eq:regret-theorem-first-term-1}
\end{align}
where $\tau_{s,a}^k:=\inf\{t:N_{s,a}^{(t)}=k\}$ and the second inequality holds because of Lemma \ref{le:concentration-empirical}. 

For Eq.~\eqref{eq:second-term-regret-1}, we have
\begin{align*}
    &\E\left[\sum_{t=0}^T \mathbb{I}\left\{a_{t}=a,s_t=s,\mathrm{LCB}_{s,a}^{(t)}\leq \gamma\right\}\right]\\&=\E\left[\sum_{t=0}^T \mathbb{I}\left\{a_t=a,s_t=s,\mathrm{LCB}_{s,a}^{(t)}\leq\gamma,\exists a'\in\mathcal{G}_s,\mathrm{LCB}_{s,a'}^{(t)}\leq \gamma\right\}\right]\\
    &\leq \E\left[\sum_{t=0}^T\sum_{a'\in\mathcal{G}_s}\mathbb{I}\left\{a_t=a,s_t=s,\mathrm{LCB}_{s,a}^{(t)}\leq\gamma,\mathrm{LCB}_{s,a'}^{(t)}\leq \gamma\right\}\right]\\
    &=\sum_{a'\in\mathcal{G}_s}\E\left[\sum_{t=0}^T\mathbb{I}\left\{a_t=a,s_t=s,\mathrm{LCB}_{s,a}^{(t)}\leq\gamma,\mathrm{LCB}_{s,a'}^{(t)}\leq \gamma\right\}\right]
\end{align*}
For each $a'\in\mathcal{G}_s$, we have
\begin{align*}
    \E\left[\sum_{t=0}^T\mathbb{I}\left\{a_t=a,s_t=s,\mathrm{LCB}_{s,a}^{(t)}\leq\gamma,\mathrm{LCB}_{s,a'}^{(t)}\leq \gamma\right\}\right]
    &=\E\left[\sum_{t=0}^T\mathbb{I}\left\{a_{t}=a',s_t=s,\mathrm{LCB}_{s,a}^{(t)}\leq\gamma,\mathrm{LCB}_{s,a'}^{(t)}\leq \gamma\right\}\right]\\
    &\leq \E\left[\sum_{t=0}^T\mathbb{I}\left\{a_{t}=a',s_t=s,\mathrm{LCB}_{s,a'}^{(t)}\leq \gamma\right\}\right]\\
    &=\E\left[\sum_{k=1}^\infty \mathbb{I}\left\{\mathrm{LCB}_{s,a'}^{(\tau_{s,a'}^k-1)}\geq\gamma,\tau_{s,a'}^k\leq T\right\}\right]\\
    &\leq \E\left[\sum_{k=1}^T \mathbb{I}\left\{\hat{r}_{s,a'}^\mathrm{1}{(\tau_{s,a'}^k-1)}\geq\gamma+\sqrt{\frac{3\log\left(N_{s,a'}^{\left(\tau_{s,a'}^k\right)}\right)}{N_{s,a'}^{\left(\tau_{s,a'}^k\right)}}}\right\}\right]\\
    &\leq u+\E\left[\sum_{k=u+1}^T \mathbb{I}\left\{\hat{r}_{s,a'}^\mathrm{1}{(\tau_{s,a'}^k-1)}\geq\gamma+\sqrt{\frac{3\log\left(N_{s,a'}^{\left(\tau_{s,a'}^k\right)}\right)}{N_{s,a'}^{\left(\tau_{s,a'}^k\right)}}}\right\}\right]\\
    &\leq u+\sum_{k=u+1}^T\frac{1}{k^2}\\
    &\leq u+2,
\end{align*}
where $u=\frac{12}{\left(\Delta^\mathrm{1}_+\right)^2}\left(\log\left(\frac{e}{\left(\Delta^\mathrm{1}_+\right)^6}\right)+\log\log\left(\frac{1}{\left(\Delta^\mathrm{1}_+\right)^6}\right)\right)$. We should note that such choice of $u$ satisfies $\sqrt{3\log(u)/u}\leq \Delta^\mathrm{1}_+$.
Therefore, Eq.~\eqref{eq:second-term-regret} can be upper bounded as
\begin{align}
    &\E\left[\sum_{t=0}^T \mathbb{I}\left\{a_t=a,s_t=s,\mathrm{LCB}_{s,a}^{(t)}\leq \gamma\right\}\right]\notag\\
    &=\sum_{a'\in\mathcal{G}_s}\E\left[\sum_{t=0}^T\mathbb{I}\left\{a_t=a,s_t=s,\mathrm{LCB}_{s,a}^{(t)}\leq\gamma,\mathrm{LCB}_{s,a'}^{(t)}\leq \gamma\right\}\right]\notag\\
    &\leq A(u+2).\label{eq:regret-theorem-second-term-1}
\end{align}
Combining Eq.~\eqref{eq:regret-theorem-first-term-1} and Eq.~\eqref{eq:regret-theorem-second-term-1} will give us the desired result.

Further assume that $C_{\text{diff}}\left|\Delta^\mathrm{1}_+\right|\geq \Delta_{s,a}^\mathrm{1}$, for all $\Delta_{s,a}^\mathrm{1}>0$. Then we have
\begin{align*}
    \mathrm{R}(\pi)&\leq \sum_{\Delta_{s,a}^\mathrm{1}\leq \Delta} \Delta_{s,a}^\mathrm{1}\E\left[T_{s,a}\right]+\sum_{\Delta_{s,a}^\mathrm{1} \geq \Delta} \frac{24C_{\text{diff}}}{\Delta^\mathrm{1}_+}\log\left(\frac{e}{\left(\Delta^\mathrm{1}_+\right)^6}\right)+\sum_{\Delta_{s,a}^\mathrm{1}>0}(2A+2)\Delta_{s,a}^\mathrm{1}\\
    &\leq \Delta T+SA24C_{\text{diff}}\frac{1}{\Delta}\log\left(\frac{e}{\Delta^6}\right)+\sum_{\Delta_{s,a}^\mathrm{1}>0}(2A+2)\Delta_{s,a}^\mathrm{1}.
\end{align*}
Take $\Delta=\sqrt{\frac{24C_{\text{diff}}SA}{T}}$ will give us the desired result.
\end{proof}
\subsection{Proof of Theorem~\ref{th:k-step-regret}}\label{appendix:proof-of-k-step}
\begin{proof}
   The same as proof of Theorem~\ref{th:one-step-regret}, the proof will proceed by first showing that the cumulative regret can be decomposed into per-timestep regret $\text{Regret}_t(\bm{\pi},\mathcal{H}_t)$. Then we will use the standard technique to bound the number of times suboptimal actions are given. We make key modifications: 1) \emph{Regret from exploration}: choosing
% We show that with 
$\epsilon_t=\mathcal{O}(\sqrt{N_{s_{t-1},a_{t-1}}^{(t)}})$ ensures that the cumulative regret due to entering the exploration phase is at most
% chosen to be of order $\sqrt{N_{s_{t-1},a_{t-1}}^{(t)}}$, there will be at most 
$K\sqrt{SAT}$;
% times that regret is caused by entering the exploration phase. 
2) \emph{Sufficient estimation samples}: the same choice of
% We show that such choice of
$\epsilon_t$ ensures that for each state-action pair $s,a$, at least $\frac{K}{2}\sqrt{SAN_{s,a}^{(T)}}$ times picking actions to estimate $r_{s,a}^\steps$.
Recall that the
% As we mentioned before Theorem~\ref{th:k-step-regret}, we estimate
K-step lookahead reward comprises
% by separating it into 
a 1-step lookahead reward and a K$-1$-step lookahead reward.
Accurate estimation of the latter
% Estimating 
% K$-1$-step lookahead reward 
is ensured by
% can be achieved by using $\mathrm{ALG}_{\step}$ in
Assumption~\ref{assumption:low-regret-algorithm}. 
% Therefore, w
We show that $\frac{K}{2}\sqrt{SAN_{s,a}^{(T)}}$ runs of
% times of running 
$\mathrm{ALG}_\step$ suffices to obtain
% is enough to guarantee
an accurate estimation of $r^\steps$. 
   
   Using the same method, we have the cumulative regret is upper bounded by
    \begin{equation*}
        \sum_{\Delta_{s,a}^\steps>0}\Delta_{s,a}^\steps\E\left[\sum_{t=0}^T \mathbb{I}\left\{a_{t}=a,s_t=s\right\}\right].
    \end{equation*}

    We have
\begin{align}
     \sum_{\Delta_{s,a}^\steps>0}\Delta_{s,a}^\steps\E\left[\sum_{t=0}^T \mathbb{I}\left\{a_{t}=a,s_t=s\right\}\right]&= \sum_{\Delta_{s,a}^\steps>0}\Delta_{s,a}^\steps\E\left[\sum_{t=0}^T \mathbb{I}\left\{a_{t}=a,s_t=s,\text{(K-1)-step}\right\}\right]\label{eq:time-k-1}\\ &+\sum_{\Delta_{s,a}^\steps>0}\Delta_{s,a}^\steps\E\left[\sum_{t=0}^T \mathbb{I}\left\{a_{t}=a,s_t=s,\text{K-step}\right\}\right]\label{eq:time-k}.
\end{align}
We will first bound Eq.~\eqref{eq:time-k-1}:
\begin{align}
    \sum_{\Delta_{s,a}^\steps>0}\Delta_{s,a}^\steps\E\left[\sum_{t=0}^T \mathbb{I}\left\{a_{t}=a,s_t=s,\text{(K-1)-step}\right\}\right]&=\sum_{t=0}^T\E\left[\sum_{\Delta_{s,a}^\steps>0}\Delta_{s,a}^\steps\mathbb{I}\left\{a_{t}=a,s_t=s,\text{(K-1)-step}\right\}\right]\notag\\
    &\leq \sum_{t=0}^T C_{\text{diff}}\Delta^\steps\E\left[\mathbb{I}\{\text{(K-1)-step}\}\right]\notag\\
    &\leq (K-1)\sum_{t=0}^TC_{\text{diff}}\Delta^\steps\frac{1}{\Delta^\step\sqrt{N_{s_{t-1},a_{t-1}}^{(t)}}}\notag\\
    &\leq (K-1)C_{\text{diff}}\sum_{s,a} \sqrt{N_{s,a}^{(T)}}\notag\\
    &\leq (K-1)C_{\text{diff}}\sqrt{SAT},\label{eq:upper-time-k-1}
\end{align}
where the last inequality holds because of Cauchy-Schwarz.

Then we will bound Eq.~\eqref{eq:time-k}
    We have
\begin{align}
    \E\left[\sum_{t=0}^T \mathbb{I}\left\{a_{t}=a,s_t=s,\text{K-step}\right\}\right]&=\E\left[\sum_{t=0}^T \mathbb{I}\left\{a_{t}=a,s_t=s,\mathrm{LCB}_{s,a}^{(t)}\geq\gamma\right\}\right]\label{eq:first-term-regret}\\&+\E\left[\sum_{t=0}^T \mathbb{I}\left\{a_{t}=a,s_t=s,\mathrm{LCB}_{s,a}^{(t)}\leq\gamma\right\}\right]\label{eq:second-term-regret}.
\end{align}
We will bound the two terms separately. For Eq.~\eqref{eq:first-term-regret}, we have
\begin{align}
    &\E\left[\sum_{t=0}^T \mathbb{I}\left\{a_{t}=a,s_t=s,\mathrm{LCB}_{s,a}^{(t)}\geq\gamma\right\}\right]\notag\\&=\E\left[\sum_{k=1}^\infty \mathbb{I}\left\{\mathrm{LCB}_{s,a}^{(\tau_{s,a}^k-1)}\geq\gamma,\tau_{s,a}^k\leq T\right\}\right]\notag\\
    &\leq \E\left[\sum_{k=1}^T \mathbb{I}\left\{\mathrm{LCB}_{s,a}^{(\tau_{s,a}^k-1)}\geq\gamma,N_{s,a,\step}^{(\tau_{s,a}^k-1)}\geq \max\left\{\frac{\sqrt{k-1}}{2\Delta^\steps},2\sqrt{k-1}\right\}\right\}\right]\notag\\&+\E\left[\sum_{k=1}^T\mathbb{I}\left\{N_{s,a,\step}^{(\tau_{s,a}^k-1)}\leq \max\left\{\frac{\sqrt{k-1}}{2\Delta^\steps},2\sqrt{k-1}\right\}\right\}\right]\label{eq:time-lcb-large}
\end{align}
where $\tau_{s,a}^k:=\inf\{t:N_{s,a}^{(t)}=k\}$. Since $N_{s,a,\step}^{(\tau_{s,a}^k-1)}=\sum_{l=1}^{k-1} \mathbb{I}\left\{\epsilon\leq \frac{1}{\min\{\Delta^\steps,1/2\}\sqrt{l}}\right\}$, where $\epsilon\sim\mathrm{Uniform}(0,1)$, we have the second term is upper bounded by $\frac{8}{\left(\Delta^\steps\right)^2}$ because of Chernoff bound. Next, we will bound the first term of Eq.~\ref{eq:time-lcb-large}.
\begin{align}
&\E\left[\sum_{k=1}^T \mathbb{I}\left\{\mathrm{LCB}_{s,a}^{(\tau_{s,a}^k-1)}\geq\gamma,N_{s,a,\step}^{(\tau_{s,a}^k-1)}\geq \max\left\{\frac{\sqrt{k-1}}{2\Delta^\steps},2\sqrt{k-1}\right\}\right\}\right]\notag\\&=\E\left[\sum_{k=1}^T \mathbb{I}\left\{\mathrm{LCB}_{s,a}^{(\tau_{s,a}^k-1)}\geq\gamma,N_{s,a,\step}^{(\tau_{s,a}^k-1)}\geq \max\left\{\frac{\sqrt{k-1}}{2\Delta^\steps},2\sqrt{k-1}\right\},\left|\E\left[\hat{r}_{s,a}^{\step}{(\tau_{s,a}^k-1)}\right]-r^\step_{s,a_{\step,s}^*}\right|\leq \sqrt{\frac{C_{K-1}\log(T)}{N_{s,a,\step}^{(\tau_{s,a}^k-1)}}}\right\}\right]\label{eq:time-regret-low}\\
&+\E\left[\sum_{k=1}^T \mathbb{I}\left\{\left|\E\left[\hat{r}_{s,a}^{\step}{(\tau_{s,a}^k-1)}\right]-r^\step_{s,a_{\step,s}^*}\right|\geq \sqrt{\frac{C_{K-1}\log(T)}{N_{s,a,\step}^{(\tau_{s,a}^k-1)}}}\right\}\right]\label{eq:time-regret-high}.
\end{align}
By Assumption~\ref{assumption:low-regret-algorithm}, Eq.~\eqref{eq:time-regret-high} is smaller than $\E\left[\sum_{k=1}^T 1/T\right]=1$. To ease notation, we will use $\mathcal{E}_A^k$ to denote $N_{s,a,\step}^{(\tau_{s,a}^k-1)}\geq \max\left\{\frac{\sqrt{k-1}}{2\Delta^\steps},2\sqrt{k-1}\right\}$ and $\mathcal{E}_B^k$ to denote $\left|\E\left[\hat{r}_{s,a}^{\step}{(\tau_{s,a}^k-1)}\right]-r^\step_{s,a_{\step,s}^*}\right|\leq \sqrt{\frac{C_{K-1}\log(T)}{N_{s,a,\step}^{(\tau_{s,a}^k-1)}}}$. For Eq.~\eqref{eq:time-regret-low}, we have
\begin{align}
   &\E\left[\sum_{k=1}^T\mathbb{I}\left\{\mathrm{LCB}_{s,a}^{(\tau_{s,a}^k-1)}\geq\gamma, \mathcal{E}_A^k,\mathcal{E}_B^k\right\}\right] \\&=\E\left[\sum_{k=1}^T\mathbb{I}\left\{\mathrm{LCB}_{s,a}^{(\tau_{s,a}^k-1)}-r_{s,a}-\E\left[\hat{r}_{s,a}^{\step}(\tau_{s,a}^k-1)\right]\geq \gamma-r_{s,a}-r^\step_{s,a_{\step,s}^*}+r^\step_{s,a_{\step,s}^*}-\E\left[\hat{r}_{s,a}^{\step}(\tau_{s,a}^k-1)\right],\mathcal{E}_A^k,\mathcal{E}_B^k\right\}\right]\notag\\
   &\leq \frac{64C_{K-1}^2\log(T)}{\left(\Delta^\steps\right)^2}+\E\left[\sum_{k=64C_{K-1}^2\log(T)/\left(\Delta^\steps\right)^2}^T\mathbb{I}\left\{\mathrm{LCB}_{s,a}^{(\tau_{s,a}^k-1)}-r_{s,a}-\E\left[\hat{r}_{s,a}^{\step}(\tau_{s,a}^k-1)\right]\geq \frac{\Delta^\steps}{2},\mathcal{E}_A^k\right\}\right]\notag\\
   &\leq \frac{64C_{K-1}^2\log(T)}{\left(\Delta^\steps\right)^2}+\E\left[\sum_{k=1}^T\mathbb{I}\left\{\hat{r}_{s,a}^{\mathrm{1}}(\tau_{s,a}^k-1)-r_{s,a}\geq \frac{\Delta^\steps}{4}+\sqrt{\frac{3\log(N_{s,a}^{(\tau_{s,a}^k-1)})}{N_{s,a}^{(\tau_{s,a}^k-1)}}}\right\}\right]\notag\\&+\E\left[\sum_{k=1}^T\mathbb{I}\left\{\hat{r}_{s,a}^{\step}(\tau_{s,a}^k-1)-\E\left[\hat{r}_{s,a}^{\step}(\tau_{s,a}^k-1)\right]\geq \frac{\Delta^\steps}{4}+\sqrt{\frac{3\log(N_{s,a,\step}^{(\tau_{s,a}^k-1)})}{N_{s,a,\step}^{(\tau_{s,a}^k-1)}}},\mathcal{E}_A^k\right\}\right]\notag\\
    &\leq\frac{64C_{K-1}^2\log(T)}{\left(\Delta^\steps\right)^2}+\sum_{k=1}^T\exp\left(-N_{s,a}^{\left(\tau_{s,a}^k\right)}\left(\Delta^\steps\right)^2+\log\log\left(N_{s,a}^{\left(\tau_{s,a}^k\right)}\right)-3\log\left(N_{s,a}^{\left(\tau_{s,a}^k\right)}\right)\right)\notag\\&+\exp\left(-2\sqrt{k}\left(\Delta^\steps\right)^2+\log\log\left(2\sqrt{k}\right)-3\log\left(2\sqrt{k}\right)\right)\notag\\
    &\leq \frac{64C_{K-1}^2\log(T)}{\left(\Delta^\steps\right)^2}+\sum_{k=1}^T \frac{1}{k^2}\exp\left(-k\left(\Delta^\steps\right)^2\right)+\frac{1}{2k}\notag\\
    &\leq \frac{128C_{K-1}^2\log(T)}{\left(\Delta^\steps\right)^2},\label{eq:regret-theorem-first-term}
\end{align}
where the thrid last inequality holds because of Lemma~\ref{le:concentration-empirical}.

For Eq.~\eqref{eq:second-term-regret}, we have
\begin{align*}
    &\E\left[\sum_{t=0}^T \mathbb{I}\left\{a_{t}=a,s_t=s,\mathrm{LCB}_{s,a}^{(t)}\leq \gamma\right\}\right]\\&=\E\left[\sum_{t=0}^T \mathbb{I}\left\{a_{t}=a,s_t=s,\mathrm{LCB}_{s,a}^{(t)}\leq\gamma,\exists(a')\in\mathcal{G_s},\mathrm{LCB}_{s,a'}^{(t)}\leq \gamma\right\}\right]\\
    &\leq \sum_{a'\in\mathcal{G}_s}\E\left[\sum_{t=0}^T\mathbb{I}\left\{a_{t}=a,s_t=s,\mathrm{LCB}_{s,a}^{(t)}\leq\gamma,\mathrm{LCB}_{s,a'}^{(t)}\leq \gamma\right\}\right]\\
    &=\sum_{a'\in\mathcal{G}_s}\E\left[\sum_{t=0}^T\mathbb{I}\left\{a_{t}=a',s_t=s,\mathrm{LCB}_{s,a}^{(t)}\leq\gamma,\mathrm{LCB}_{s,a'}^{(t)}\leq \gamma\right\}\right]\\
    &\leq \sum_{a'\in\mathcal{G}_s}\E\left[\sum_{t=0}^T\mathbb{I}\left\{a_{t}=a',s_t=s,\mathrm{LCB}_{s,a'}^{(t)}\leq \gamma\right\}\right]
\end{align*}
For each $a'\in\mathcal{G}_s$, we have
\begin{align}
    \E\left[\sum_{t=0}^T\mathbb{I}\left\{a_{t}=a',s_t=s,\mathrm{LCB}_{s,a'}^{(t)}\leq \gamma\right\}\right]&\leq \E\left[\sum_{k=1}^T\mathbb{I}\left\{\mathrm{LCB}_{s,a'}^{(\tau_{s,a'}^k-1)}\leq\gamma,\mathcal{E}_A^k,\mathcal{E}_B^k\right\}\right]+\frac{8}{(\Delta^\steps)^2},\label{eq:regret-theorem-second}
\end{align}
where the inequality holds similar to bounding Eq.~\eqref{eq:first-term-regret}. 
\begin{align}
    &\E\left[\sum_{k=1}^T\mathbb{I}\left\{a_{t}=a',s_t=s,\mathrm{LCB}_{s,a'}^{(\tau_{s,a'}^k)}\leq\gamma,\mathcal{E}_A^k,\mathcal{E}_B^k\right\}\right]\notag\\
    &=\E\left[\sum_{k=1}^T\mathbb{I}\left\{\mathrm{LCB}_{s,a'}^{(\tau_{s,a'}^k-1)}-r_{s,a'}-\E\left[\hat{r}_{s,a}^{\step}(\tau_{s,a'}^k-1)\right]\leq \gamma-r_{s,a'}-r^\step_{s,a_{\step,s}^*}+r^\step_{s,a_{\step,s}^*}-\E\left[\hat{r}_{s,a'}^{\step}(\tau_{s,a'}^k-1)\right],\mathcal{E}_A^k,\mathcal{E}_B^k\right\}\right]\notag\\
    &\leq u+\E\left[\sum_{k=u+1}^T\mathbb{I}\left\{\mathrm{LCB}_{s,a'}^{(\tau_{s,a'}^k-1)}-r_{s,a'}-\E\left[\hat{r}_{s,a'}^{\step}(\tau_{s,a'}^k-1)\right]\leq \frac{\Delta^\steps}{2},\mathcal{E}_A^k\right\}\right]\notag\\
   &\leq u+\E\left[\sum_{k=1}^T\mathbb{I}\left\{\hat{r}_{s,a'}^{\mathrm{1}}(\tau_{s,a'}^k-1)-r_{s,a'}\leq -\frac{\Delta^\steps}{4}+\sqrt{\frac{3\log(N_{s,a'}^{(\tau_{s,a'}^k-1)})}{N_{s,a'}^{(\tau_{s,a'}^k-1)}}}\right\}\right]\notag\\&+\E\left[\sum_{k=1}^T\mathbb{I}\left\{\hat{r}_{s,a'}^{\step}(\tau_{s,a'}^k-1)-\E\left[\hat{r}_{s,a'}^{\step}(\tau_{s,a'}^k-1)\right]\leq -\frac{\Delta^\steps}{4}+\sqrt{\frac{3\log(N_{s,a',\step}^{(\tau_{s,a'}^k-1)})}{N_{s,a',\step}^{(\tau_{s,a'}^k-1)}}},\mathcal{E}_A^k\right\}\right]\notag\\
    &\leq u+\sum_{k=1}^T\exp\left(-N_{s,a'}^{\left(\tau_{s,a'}^k\right)}\left(\Delta^\steps\right)^2+\log\log\left(N_{s,a'}^{\left(\tau_{s,a'}^k\right)}\right)-3\log\left(N_{s,a'}^{\left(\tau_{s,a'}^k\right)}\right)\right)\notag\\&+\exp\left(-2\sqrt{k}\left(\Delta^\steps\right)^2+\log\log\left(2\sqrt{k}\right)-3\log\left(2\sqrt{k}\right)\right)\notag\\
    &\leq u+\sum_{k=1}^T \frac{1}{k^2}\exp\left(-k\left(\Delta^\steps\right)^2\right)+\frac{1}{2k}\notag\\
    &\leq 2u,\label{eq:regret-theorem-second-term}
\end{align}
where $u=\max\left\{\frac{64C_{K-1}^2\log(T)}{\left(\Delta^\steps\right)^2},\frac{144}{\left(\Delta^\steps\right)^2}\left(\log\left(\frac{e}{\left(\Delta^\steps\right)^6}\right)+\log\log\left(\frac{1}{\left(\Delta^\steps\right)^6}\right)\right)\right\}$. We should note that such choice of $u$ satisfies $\sqrt{3\log(u)/u}\leq \Delta^\steps/4$ and $\sqrt{2\Delta^\steps C_{K-1}\log(T)/\sqrt{(u-1)}}\leq \Delta^\steps/2$.

Combining Eq.~\eqref{eq:upper-time-k-1}, Eq.~\eqref{eq:regret-theorem-first-term}, Eq.~\eqref{eq:regret-theorem-second} and Eq.~\eqref{eq:regret-theorem-second-term}, we have the regret is upper bounded by $(K-1)C_{\text{diff}}\sqrt{SAT}+\frac{256C_{\text{diff}}C_{k-1}^2SA\log(T)}{\Delta^\steps}$. Then we have $\mathcal{R}\leq \min\{(K-1)C_{\text{diff}}\sqrt{SAT}, 16C_{\text{diff}}C_{K-1}\sqrt{SAT\log(T)}\}$
\end{proof}
\section{Experiment Details}
\subsection{Environment Details}
\paragraph{Synthetic MDP:}\label{paragraph:synthetic}
We randomly generate 1000 instances with 10 states and 5 actions and generate 1000 instances with 100 states and 25 actions. The mean reward $R(s,a)$ is generated according to $\mathrm{Gamma}(0.5,1)$ and the reward is normal with mean $R(s,a)$ and variance 0.5. For 10 state case, the transition probabilities are generated according to $\mathrm{Gamma}(0.1,10)$ and then normalized. For 100 state case, the transition probabilities are generated according to $\mathrm{Gamma}(0.01,1000)$.
\paragraph{JumpRiverSwim \citep{wei2020model}:}\label{paragraph:jumpriverswim} JumpRiverSwim models a swimmer who can choose to swim left or right in a river. The states are arranged in a chain and the swimmer starts from the leftmost state $s=0$. We will denote the rightmost state as $S$. At state 0, if the swimmer chooses to swim left, they will go to any arbitrary state with probability 0.01 and they will stay at state 0 with probability $1-0.01$. If the swimmer chooses to swim right, they will go to any arbitrary state with probability 0.01, stay at state 0 with probability $0.7+0.01/(S+1)$ and go to the right state with probability $0.3-0.01$. At state 1, if the swimmer chooses to swim right, they will go to any arbitrary state with probability 0.01, go to the left state with probability $0.7+0.01/(S+1)$ and stay at $S$ with probability $0.3-0.01$. If the swimmer chooses to swim left, they will go to any arbitrary state with probability 0.01 and they will go to the left state with probability $1-0.01$.  At any other states, if the  swimmer chooses to swim left, they will go to any arbitrary state with probability 0.01 and they will go to the left state with probability $1-0.01$. If they choose to swim right, they will go to any arbitrary state with probability 0.01, go to the left state with probability $0.6+0.01/(S+1)$, go to  the right state with probability $0.3-0.01$ and stay at current state with probability $0.1+0.01/S$. The reward is all zero except at state 0 and state $S$. Specifically $r(0,0)=0.2,r(S,1)=1$. 

In our experiment, we set $S$ to be $4,7,14$. 
\paragraph{FrozenLake \citep{brockman2016openaigym}:}\label{paragraph:Frozenlake} The game starts at position $[0,0]$ and the goal is to reach the end state for as many times as possible. The goal is at far extent of the game environment. At each time, the player will choose to go to one of its four neighbors. If the player is at a frozen tile, it may go along perpendicular to the intended direction with probability $1/3$. If the player reaches hole or goal, it will immediately go back to the starting state. In order to encourage avoiding holes, we set the reward when in hole be 0, the reward when on frozen tile to be 0.2 and the reward when reaches the goal to be 1. We test on a $4\times 4$ grid with the following two maps.
\begin{figure}[ht]
  \begin{center}
    \centerline{\includegraphics[scale=0.5]{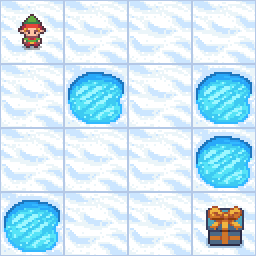}}
    \caption{
      The FrozenLake environment.
    }
    \label{fig:frozenlake}
  \end{center}
\end{figure}
\paragraph{AnyTrading \citep{gym_anytrading}:}\label{paragraph:anytrading}
This is a simulated trading environment built from real dataset. The state space is an array that contains the close price of the previous day an thus is continuous. It contained two actions: sell or buy. We use the default FOREX environment contained in \citet{gym_anytrading}. 
\subsection{Implementation Details}\label{appendix:implementation}
Since to the best our knowledge, no existing algorithms directly address non-episodic finite horizon MDP, we choose algorithms for infinite horizon average reward and episodic RL as benchmarks. We chose algorithms that are theoretically minimax optimal \citep{boone2024achieving} and have been reported empirically to work well \citep{wei2020model,jin2018q}.

Specifically, we used the following benchmarks:
\begin{itemize}
    \item UCRL2 \citep{NIPS2008_e4a6222c}: At each step, construct a confidence ball of transition probabilities using $\mathcal{L}_1$ distance. Then they will use extended value iteration (EVI) to calculate the policy that maximizes the average reward of MDP in the confidence region.
    \item KLUCRL \citep{filippi2010optimism}:  At each step, construct a confidence ball of transition probabilities using KL divergence. Then they will use extended value iteration to calculate the policy that maximizes the average reward of MDP in the confidence region.
    \item PMEVI-KLUCRL \citep{boone2024achieving}: This replaces the extended value iteration with projection mitigated extended value iteration which solves extended value iteration with a span constraint. They use the same way to construct confidence ball as KLUCRL.
    \item Optimistic Q Learning \citep{wei2020model}: This learns a Q function of discounted reward MDP with high discount factor and chooses greedily according to the Q function. As suggested in the paper, we make two choices of discount factor $\gamma$: 0.9 and 0.99. We note that since this learns discounted reward Q function, it can also be seen as discounted reward algorithml.
    \item MDP-OOMD \citep{wei2020model}: This applies policy mirror decent to average reward setting. 
    \item Q Learning \citep{jin2018q}: This learns the Q function for finite horizon MDP in episodic setting. In order to maximize the number of episodes, we let the horizon of the Q function be $1,10$. 
\end{itemize}
The implementation of the model based methods \citep{NIPS2008_e4a6222c,filippi2010optimism,boone2024achieving} and model free average reward methods \citep{wei2020model} are the same as the original implementation and the hyperparameters are the same as the ones suggested in those papers.
\subsection{Additional Experiments on Adaptively Choosing $\bm\gamma$}\label{appendix:adaptive-threshold}
In this section, we provide a heuristic of changing the threshold: starting from the lower bound of the K-step lookahead reward and then increase the threshold by $\Delta$ for state $s$ when the number of visit $N_s^{(t)}$ exceeds $\log(T)$. As shown in Figures~\ref{fig:t-synthetic} and \ref{fig:t-synthetic-100}, this heuristic (LG1T-I) allows us to start LG1T from initial threshold being 0 for all states while the overall performance could exceed the performance of starting at higher threshold on synthetic environments and FrozenLake. Same behviour happens to LG2T This makes sure that the algorithm can learn fast at the beginning and begin to learn a better policy when there is sufficient knowledge. We did not compare on other environments since Figures~\ref{fig:ablation_riverswim}-\ref{fig:ablation_trading} shows different threshold does not make big difference.
\begin{figure}
    \centering
    \includegraphics[width=\linewidth]{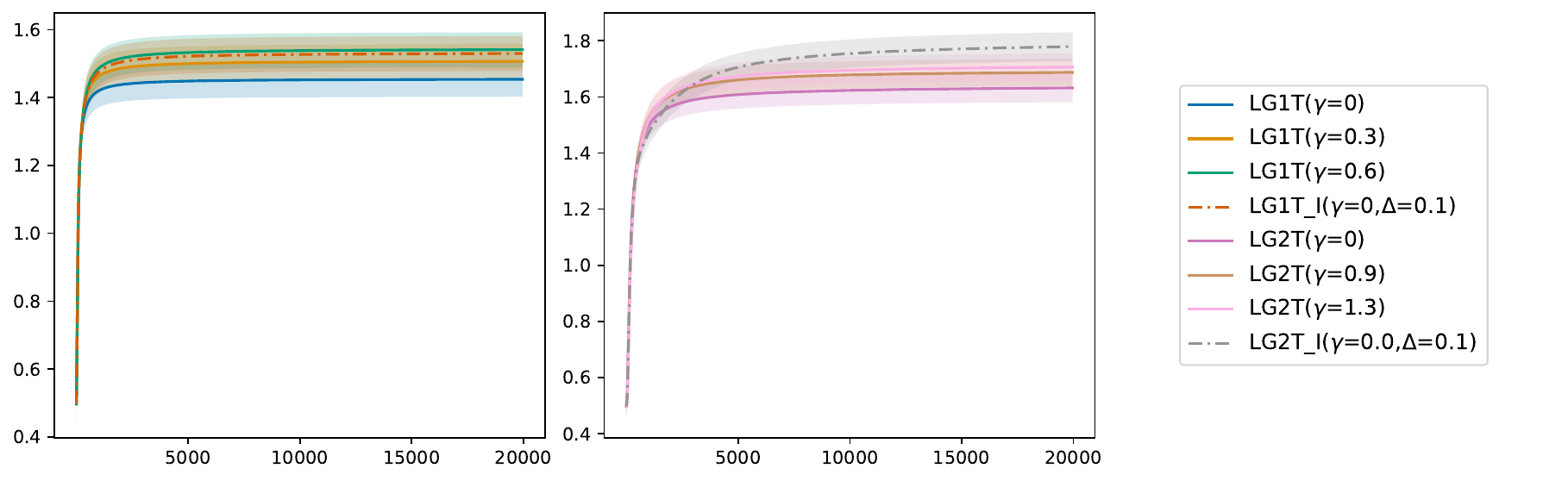}
    \caption{Comparison of adaptive choosing threshold on 10 states and 5 actions synthetic environment. Right: LG1T. Left: LG2T.}
    \label{fig:t-synthetic}
\end{figure}
\begin{figure}
    \centering
    \includegraphics[width=\linewidth]{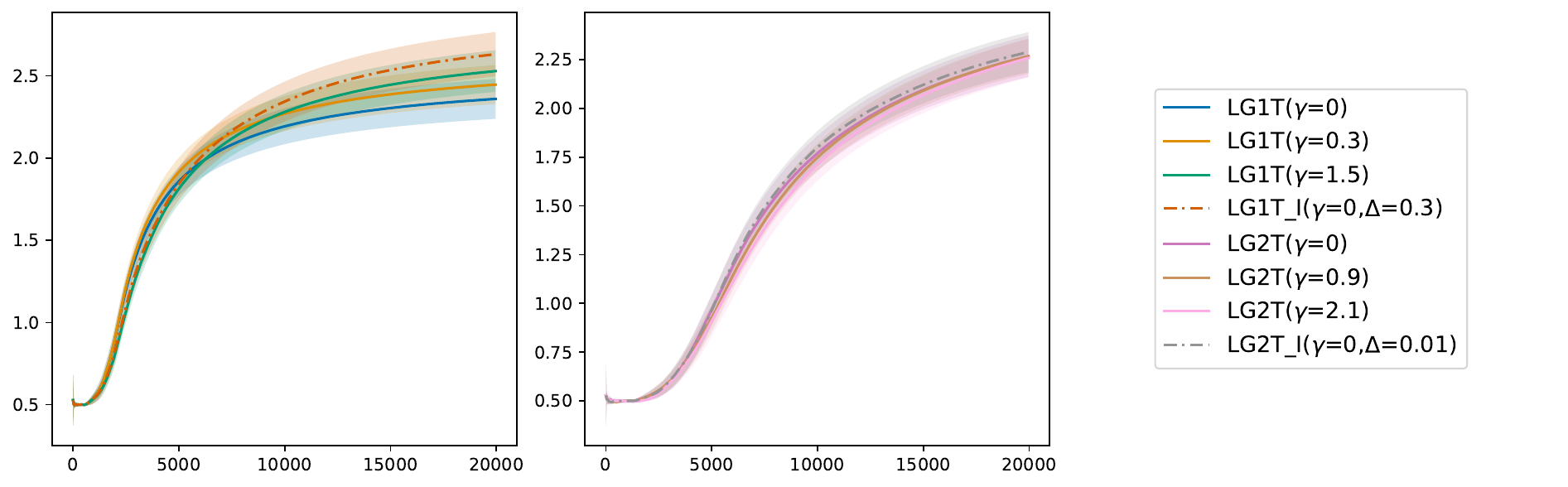}
    \caption{Comparison of adaptive choosing threshold on 100 states and 25 actions synthetic environment. Right: LG1T. Left: LG2T.}
    \label{fig:t-synthetic-100}
\end{figure}
\subsection{Additional Experiments on Adaptively Changing K}\label{appendix:adaptive-k}
In this section, we provide a heuristic of the data-driven change time of Algorithm~\ref{alg:lg1-2t}. Instead of choosing a fixed change time, we will let it be $\sqrt{SAT}$. This matches the regret bound for LG1T and serve as a sign for learning 1-step lookahead policy well. As shown in Figures~\ref{fig:k-synthetic}-\ref{fig:k-frozenlake} , this heuristic (LG1-2T-Adaptive) outperforms LG1T on all environments and outperform LG2T except the 10 state synthetic instances.
\begin{figure}
    \centering
    \includegraphics[width=\linewidth]{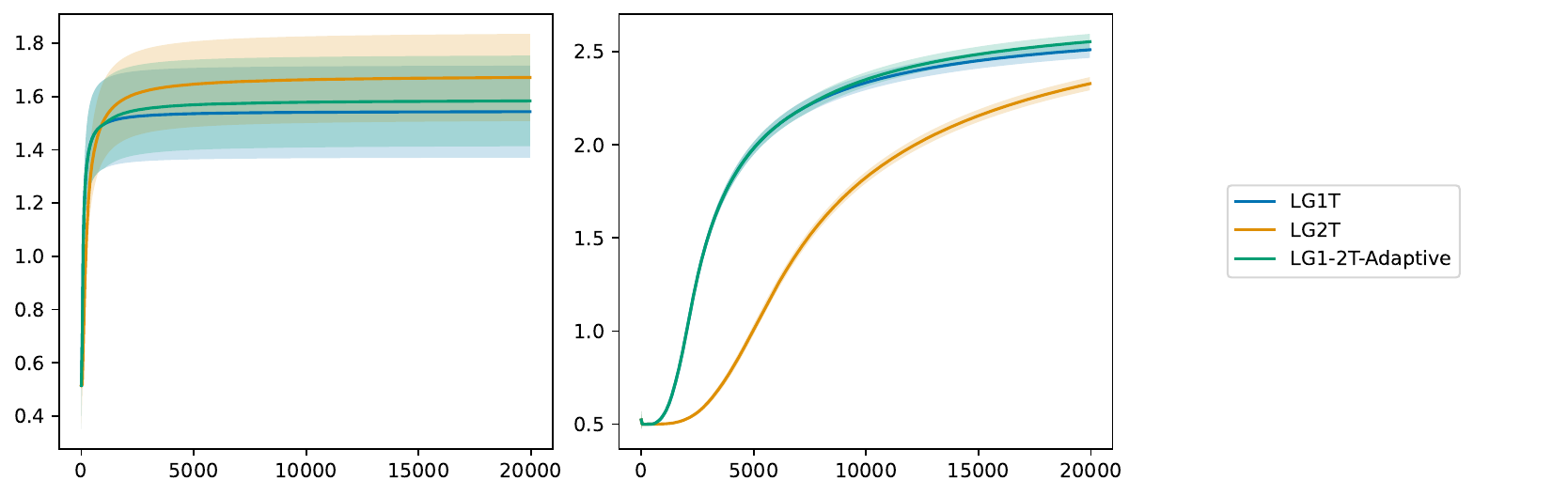}
    \caption{Comparison of adaptive choosing the changing time. Right: 10 states and 5 actions. Left: 100 states and 25 actions.}
    \label{fig:k-synthetic}
\end{figure}
\begin{figure}
    \centering
    \includegraphics[width=\linewidth]{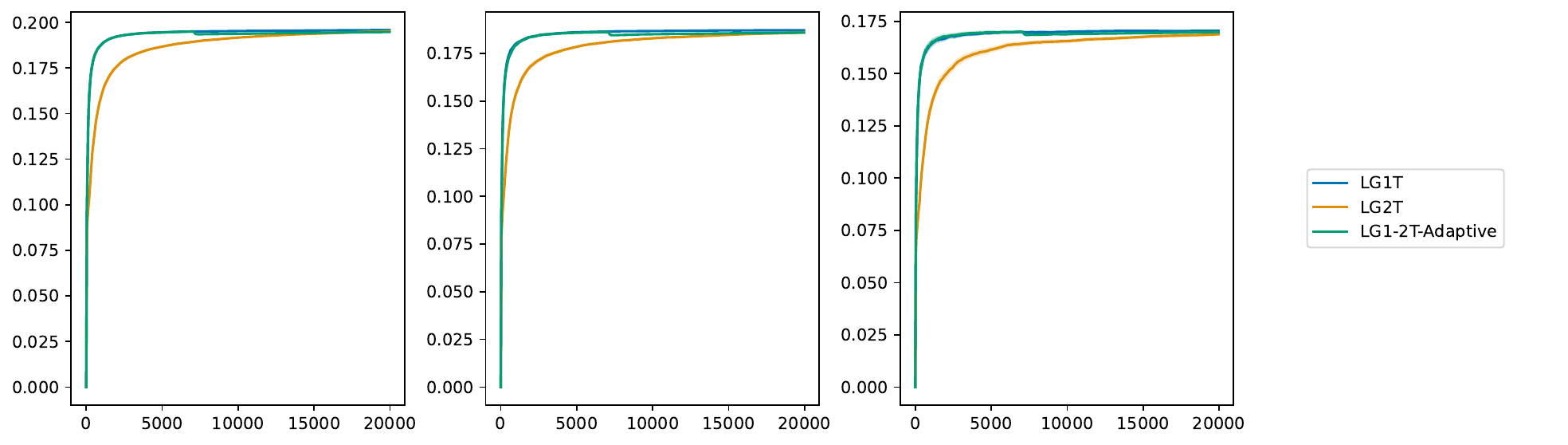}
    \caption{Comparison of adaptive choosing the changing time on JumpRiverswim. Right: 5 states. Middle: 8 states. Left: 15 states.}
    \label{fig:k-riverswim}
\end{figure}
\begin{figure}
    \centering
    \includegraphics[width=\linewidth]{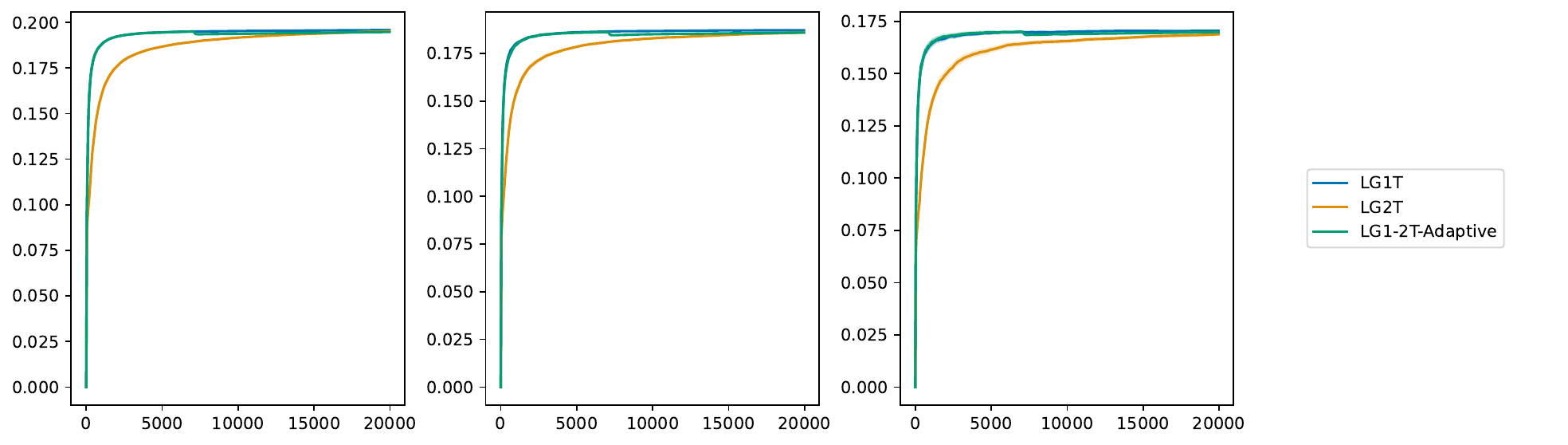}
    \caption{Comparison of adaptive choosing the changing time on FrozenLake.}
    \label{fig:k-frozenlake}
\end{figure}
\subsection{Additional Experiments on Theory Version of Algorithms~\ref{alg:'name'}, \ref{alg:k-step}}\label{appendix:theory-version}
As shown in Figures~\ref{fig:uniform-synthetic}-\ref{fig:uniform-frozen}, LG1T-Uniform with the same or smaller threshold has similar performance compared with LG1T that uses UCB when no actions' LCB is above the threshold and consistently beats the baselines on all environments we tested. LG2T-Uniform with the same or smaller threshold also has similar performance compared with LG2T on all environments. On JumpRiverswim environment, LG2T is worse than PMEVI-KLUCRL, this supports our point: when all actions' LCB is below the threshold, which can happen since the reward of many middle state is 0, using UCB can give us a better way of choosing the action with potentially higher reward. 

These results indicate the theory version of Algorithms~\ref{alg:'name'}, \ref{alg:k-step} can also be the best among all benchmarks but the implemented version can further improve the theory version.
\begin{figure}
    \centering
    \includegraphics[width=\linewidth]{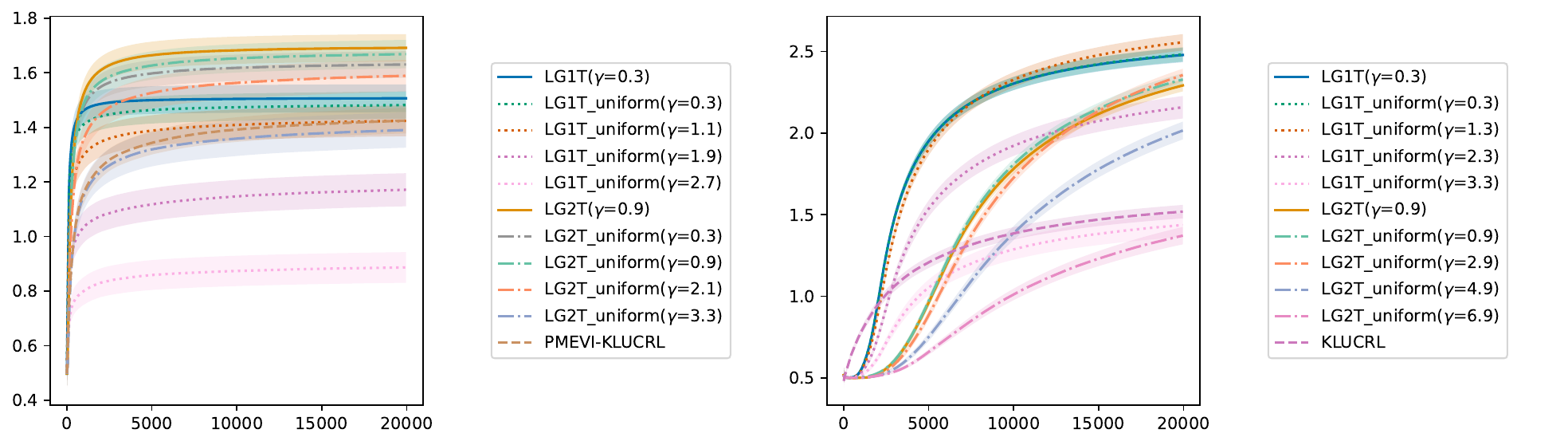}
    \caption{Comparison of theory version and implemented version on 1000 synthetic MDPs. Right: 10 states and 5 actions. Left: 100 states and 25 actions.}
    \label{fig:uniform-synthetic}
\end{figure}
\begin{figure}
    \centering
    \includegraphics[width=\linewidth]{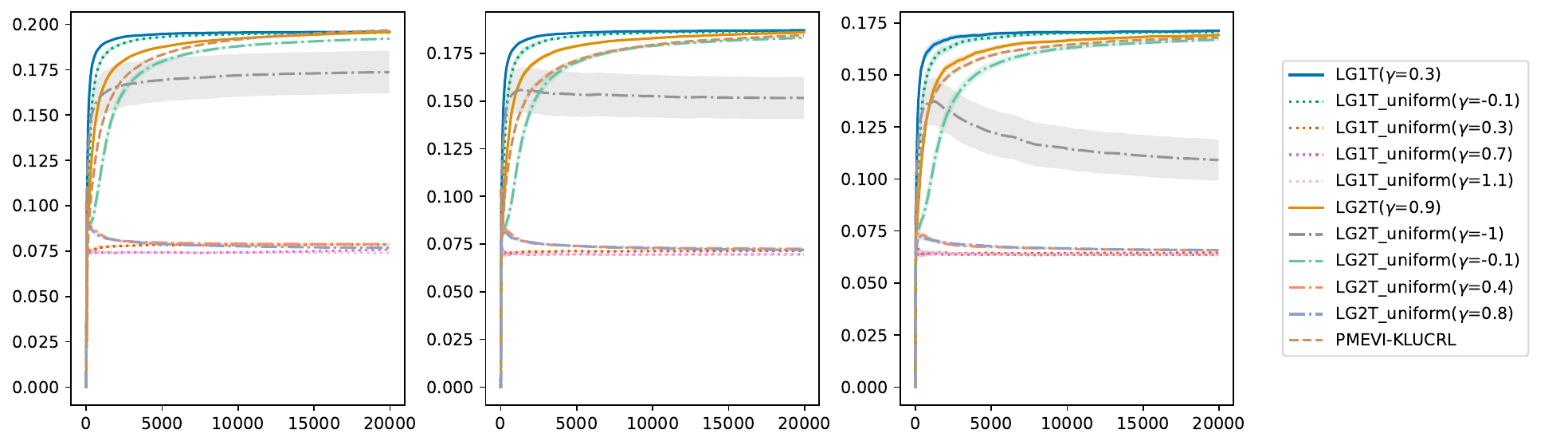}
    \caption{Comparison of theory version and implemented version on JumpRiverswim. Right: 5 states. Middle: 8 states. Left: 15 states.}
    \label{fig:uniform-riverswim}
\end{figure}
\begin{figure}
    \centering
    \includegraphics[width=\linewidth]{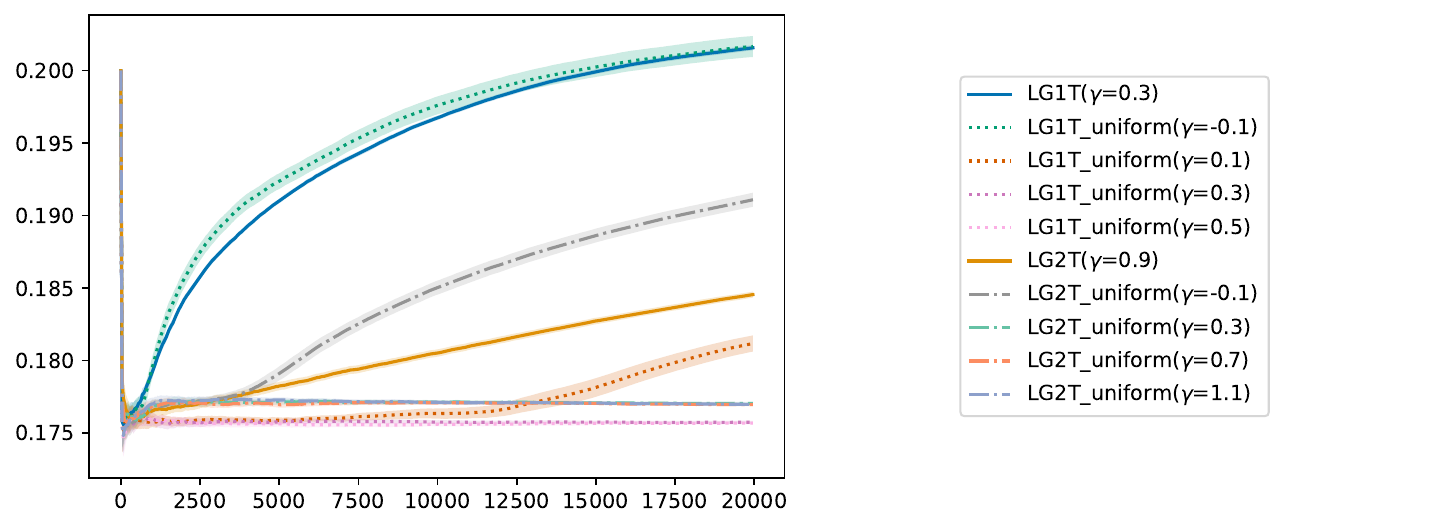}
    \caption{Comparison of theory version and implemented version on FrozenLake.}
    \label{fig:uniform-frozen}
\end{figure}
\subsection{Ablation on $\bm\gamma$}\label{appendix:ablation-threshold}
In this section, we will test the performance of our algorithms for a various choices of threshold on all instances we have tested. As shown in Figures~\ref{fig:ablation_synthetic}-\ref{fig:ablation_frozenlake}, a high threshold will lead to better asymptotic performance which matches our observation that with higher threshold, $\bm\pi^{\steps,\bm\gamma}$ will have better performance. However, we also observe that in the synthetic MDP with $S=100,A=25$, higher threshold will first be worse than lower threshold in the beginning because as we mentioned in Section~\ref{sec:online-algorithm}, higher threshold will slower the convergence. Nevertheless, as shown in the figures, the difference is small and this shows that in practice, a moderate choice of threshold can be chosen to balance between maximizing the convergence and maximizing the reward. Notably, the specific threshold used in our main experiments (Section~\ref{sec:experiments}), while not tuned for peak performance, still enables our algorithms to consistently outperform all benchmarks, underscoring the robustness of our design.

\begin{figure}[ht]
  \begin{center}
    \centerline{\includegraphics[scale=0.4]{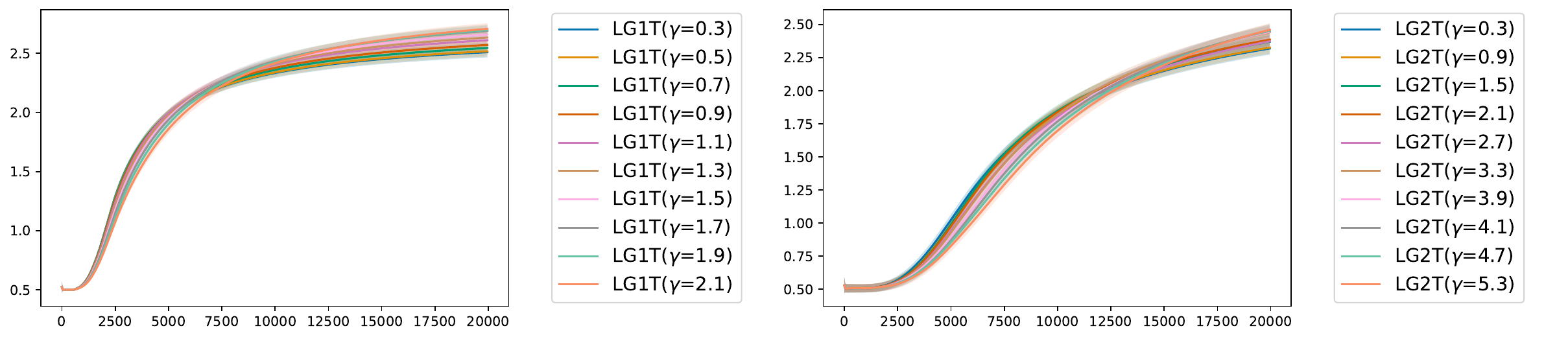}}
    \caption{
      Ablations on 1000 synthetic MDPs. $S=100,A=25$}
    \label{fig:ablation_synthetic}
  \end{center}
\end{figure}
\begin{figure}[ht]
  \begin{center}
    \centerline{\includegraphics[scale=0.4]{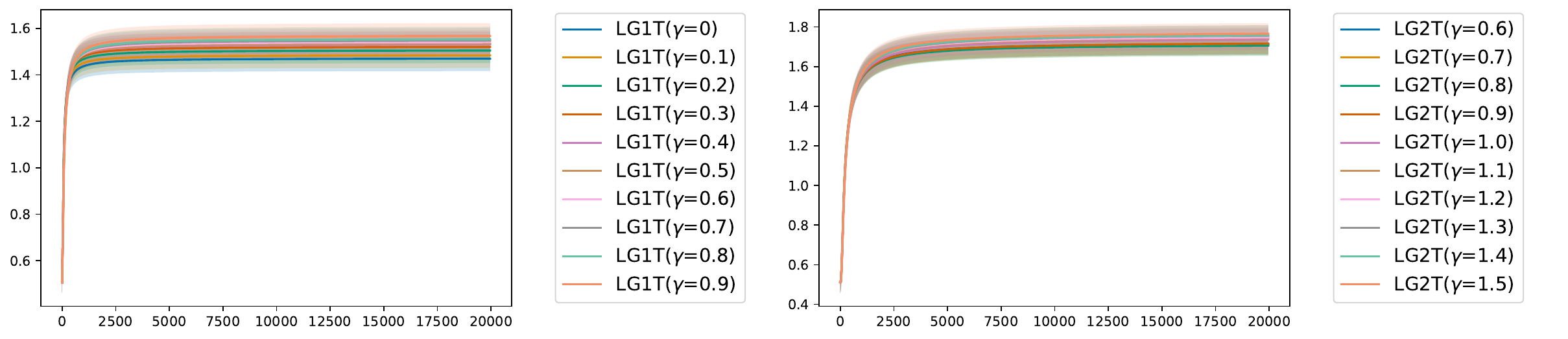}}
    \caption{
      Ablations on 1000 synthetic MDPs. $S=10,A=5$
    }
    \label{fig:ablation_synthetic_100}
  \end{center}
\end{figure}
\begin{figure}[ht]
  \begin{center}
    \centerline{\includegraphics[width=\columnwidth]{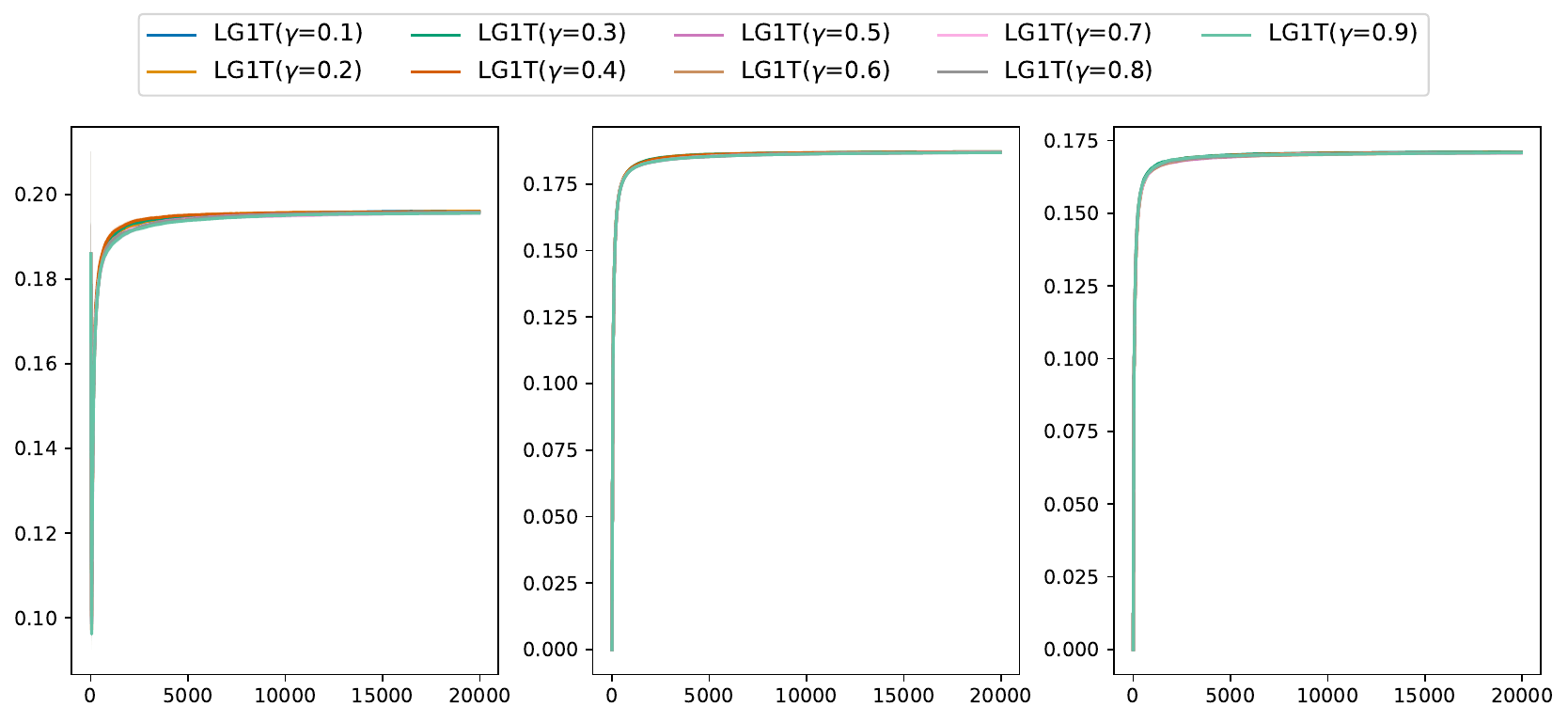}}
    \caption{
      Ablations on JumpRiverSwims. Left: $S=5$, Middle: $S=8$, Right: $S=15$. 
    }
    \label{fig:ablation_riverswim}
  \end{center}
\end{figure}
\begin{figure}[ht]
  \begin{center}
    \centerline{\includegraphics[width=\columnwidth]{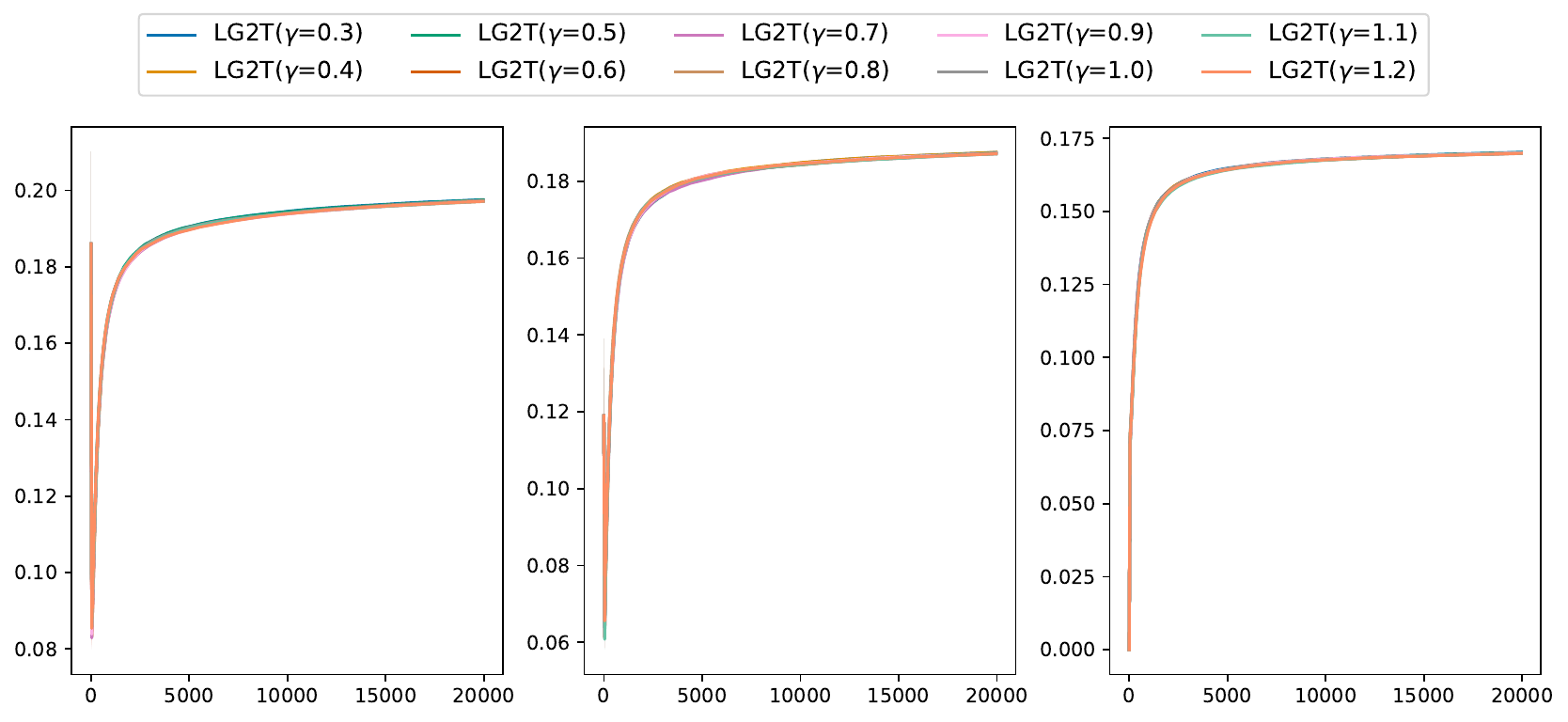}}
    \caption{
      Ablations of LG2T on JumpRiverSwims. Left: $S=5$, Middle: $S=8$, Right: $S=15$.
    }
    \label{fig:ablation_riverswim_2}
  \end{center}
\end{figure}
\begin{figure}[ht]
  \begin{center}
    \centerline{\includegraphics[width=\columnwidth]{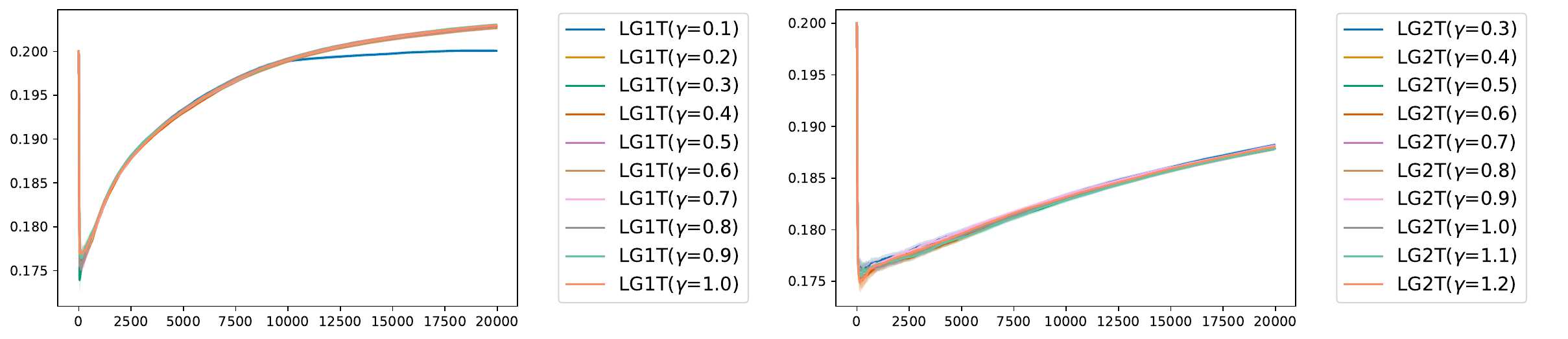}}
    \caption{
      Ablations on FrozenLake. Left: LG1T, Right: LG2T.
    }
    \label{fig:ablation_frozenlake}
  \end{center}
\end{figure}
\begin{figure}[ht]
  \begin{center}
    \centerline{\includegraphics[width=\columnwidth]{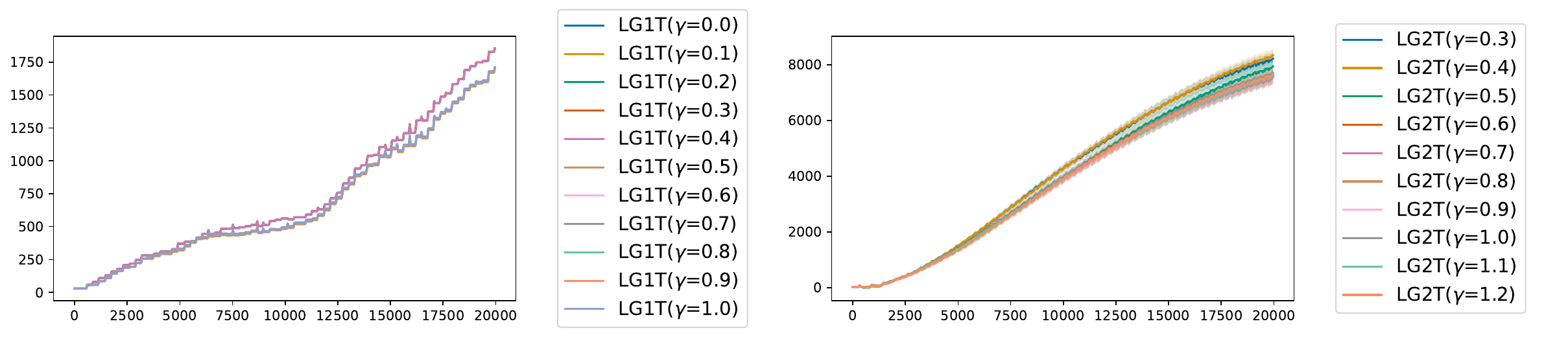}}
    \caption{
      Ablations on AnyTrading. Left: LG1T, Right: LG2T.
    }
    \label{fig:ablation_trading}
  \end{center}
\end{figure}

\section{Regret Bounds for Implemented Algorithm~\ref{alg:'name'}}\label{appendix:regret-modified}
\begin{theorem}\label{th:one-step-regret-implemented}
Under Assumption~\ref{assumption:exists-a-good-action} with $K=1$, the regret $\mathcal{R}^{\bm\pi^{\mathrm{1},\bm\gamma}}$ of implemented Algorithm~\ref{alg:'name'} satisfies
% \vspace{-10pt}
    $
\mathcal{R}^{\bm\pi^{\mathrm{1},\bm\gamma}}=\mathcal{O}\left(\sqrt{SAT\log(T)}\right).
    $
\end{theorem}
\begin{proof}
    Since the only change is when no arm's LCB is higher than the threshold, we only need to bound Eq.~\eqref{eq:second-term-regret-1}. The other term will be the same as the proof of Theorem~\ref{th:one-step-regret}. We have
\begin{align*}
    &\E\left[\sum_{t=0}^T \mathbb{I}\left\{a_{t}=a,s_t=s,\mathrm{LCB}_{s,a}^{(t)}\leq \gamma\right\}\right]\\&=\E\left[\sum_{t=0}^T \mathbb{I}\left\{a_t=a,s_t=s,\mathrm{LCB}_{s,a}^{(t)}\leq\gamma,\mathrm{UCB}_{s,a}^{(t)}\geq \mathrm{UCB}_{s,a_{s,1}^*}^{(t)}\right\}\right]\\
    &\leq \E\left[\sum_{t=0}^T \mathbb{I}\left\{s_t=s,\mathrm{UCB}_{s,a}^{(t)}\geq \mathrm{UCB}_{s,a_{s,1}^*}^{(t)}\right\}\right]\\
    &\leq \sqrt{AN_s^{(T)}\log(T)},
\end{align*}
where the last inequality holds using the same technique as in \cite{lattimore2020bandit}. Then we have the regret is upper bounded by $\sum_{s}\sqrt{AN_s^{(T)}\log(T)}\leq \sqrt{SAT\log(T)}$. This completes the proof.
\end{proof}